\documentclass[twocolumn]{svjour3}          
\smartqed  
\usepackage[T1]{fontenc}    
\usepackage[utf8]{inputenc} 
\usepackage{fourier}
\usepackage{url}            
\usepackage{amsfonts}       
\usepackage{nicefrac}       
\usepackage{microtype}      
\usepackage[dvipsnames]{xcolor}         

\usepackage{graphicx}
\usepackage{times}
\usepackage{epsfig}
\usepackage{amsmath}
\usepackage{amssymb}
\usepackage{lineno}
\usepackage{cite}
\usepackage{lipsum} 
\usepackage{booktabs}
\usepackage{enumitem}
\usepackage{adjustbox}
\usepackage{tikz}
\usepackage{pifont}
\usepackage{subfigure}

\usepackage{xspace}
\usepackage{color}
\usepackage{eqparbox}
\usepackage{footnote}
\usepackage{array}
\usepackage{bm}
\usepackage{bbm}
\usepackage{multirow}
\usepackage{makecell}
\usepackage{mathtools}
\usepackage{tabularx}
\usepackage{threeparttable}
\usepackage{wrapfig,lipsum,booktabs}
\usepackage{colortbl}
\usepackage{wasysym}
\usepackage{dsfont}
\usepackage{pifont} 
\usepackage{dsfont}  
\usepackage{siunitx}
\usepackage{multicol}
\usepackage[sectionbib]{natbib}

\usepackage{marvosym}

\usepackage{algorithm}
\usepackage{algorithmicx}
\usepackage[noend]{algpseudocode}
\renewcommand{\algorithmicrequire}{\textbf{Input:}}
\renewcommand{\algorithmicensure}{\textbf{Output:}}

\usepackage[
pagebackref=false,
breaklinks=true,
letterpaper=true,
colorlinks,
bookmarks=true,
bookmarksopen=true,
bookmarksdepth=3,
citecolor=blue,
linkcolor=blue
]{hyperref}

\newcommand{\E}{\mathbb{E}}
\newcommand{\KL}{\mathrm{KL}}
\newcommand{\SKL}{D_{\text{skl}}}
\newcommand{\InF}{\mathbb{1}}
\newcommand{\softmax}{\mathrm{softmax}}

\newcommand{\positive}[1]{\textcolor{Green}{#1}}
\definecolor{violet}{HTML}{BF00FF} 
\newcommand{\RCRevised}[1]{{\color{black}#1}}

\newcommand{\pgraph}[1]{{\noindent\textbf{#1.}}}

\newcommand{\figref}[1]{Fig.~\ref{fig:#1}}
\def\ie{i.e.,~}
\def\eg{e.g.,~}
\def\cf{cf.~}
\def\etal{et~al.\xspace}

\begin{document}

\title{Correlation Information Bottleneck: Towards Adapting Pretrained Multimodal Models for Robust Visual Question Answering}



\author{Jingjing Jiang \and Ziyi Liu \and Nanning Zheng}
\authorrunning{Jiang et al.}
\institute{
Jingjing Jiang, \email{jingjingjiang2017@gmail.com}\\
Ziyi Liu, \email{liuziyi@stu.xjtu.edu.cn}\\
Nanning Zheng$^{\text{\Letter}}$, \email{nnzheng@mail.xjtu.edu.cn}\\
\at Institute of Artificial Intelligence and Robotics,\\
Xi'an Jiaotong University, Xi'an, Shaanxi, 710049, China
}


\date{Received: date / Accepted: date}

\maketitle

\begin{abstract}
Benefiting from large-scale pretrained vision language models (VLMs), the performance of visual question answering (VQA) has approached human oracles. 
However, finetuning such models on limited data often suffers from overfitting and poor generalization issues, leading to a lack of model robustness. 
In this paper, we aim to improve input robustness from an information bottleneck perspective when adapting pretrained VLMs to the downstream VQA task. 
Input robustness refers to the ability of models to defend against visual and linguistic input variations, as well as shortcut learning involved in inputs. 
Generally, the representations obtained by pretrained VLMs inevitably contain irrelevant and redundant information for a specific downstream task, resulting in statistically spurious correlations and insensitivity to input variations. 
To encourage representations to converge to a minimal sufficient statistic in multimodal learning, we propose Correlation Information Bottleneck (CIB), which seeks a tradeoff between compression and redundancy in representations by minimizing the mutual information (MI) between inputs and representations while maximizing the MI between outputs and representations. 
\RCRevised{
Moreover, we derive a tight theoretical upper bound for the mutual information between multimodal inputs and representations, incorporating different internal correlations that guide models to learn more robust representations and facilitate modality alignment. 
Extensive experiments consistently demonstrate the effectiveness and superiority of the proposed CIB in terms of input robustness and accuracy. 
} 
\end{abstract}

\section{Introduction}

Visual Question Answering (VQA) is a typical multimodal task that answers a given question based on image understanding \citep{antol2015vqa}. 
Recently, large-scale pretrained Vision-Language Models (VLMs) \citep{wang2022image,zeng2022x,wang2022ofa,yu2022coca,wang2022simvlm,li2022blip,yuan2021florence,wang2021vlmo} have advanced VQA performance to the level of human oracle. 
However, finetuning such pretrained VLMs on limited data for the downstream VQA task often leads to overfitting and poor generalization, limiting the improvement in robustness that pretrained VLMs can offer compared to the improvement in accuracy. 

In this paper, we investigate how to effectively improve input robustness when adapting pretrained VLMs to a downstream VQA task. 
Input robustness in VQA refers to the ability of models to defend against visual variations (such as question-related object removal in images \citep{agarwal2020towards}), linguistic variations (such as word substitution and sentence rephrasing in questions \citep{shah2019cycle}), and multimodal shortcut learning involved in input images and questions \citep{dancette2021beyond}. 
Practically, during the finetuning process, VQA is usually formulated as a multi-answer classification problem or a text generation problem, where pretrained multimodal transformers act as representation extractors with rich knowledge and are utilized to extract vision-language representations for answer prediction. 
As such, improving the input robustness of models essentially means obtaining more compact and task-related representations. 

To this end, we propose to improve input robustness from an information-theoretical perspective. 
The representations yielded by pretrained VLMs inevitably contain irrelevant and redundant information for the specific downstream task, which is one possible reason for poor robustness. 
Irrelevant information encourages models to learn statistically spurious correlations between representations and labels, while task-agnostic redundant information reduces the sensitivity of models to input variations. 
Both of these factors impair the input robustness of models. 
To obtain more robust and compact representations, we thus anticipate that when adapting pretrained VLMs to VQA, these pretrained VLMs can discard irrelevant and redundant information in representations while preserving task-relevant information. 
The information bottleneck principle \citep{tishby2000information} is adept at seeking a tradeoff between representation compression and redundancy. 
Motivated by this insight, we explore how to elegantly generalize the information bottleneck to find the minimal sufficient statistic for the learned representations, thereby improving the input robustness of VQA models. 

We propose Correlation Information Bottleneck (CIB) to enhance input robustness when adapting pretrained VLMs to the downstream VQA task. 
\RCRevised{Overall, by minimizing mutual information (MI) between representations and inputs while maximizing MI between representations and outputs, CIB seeks an optimal tradeoff between compression and redundancy in the representations learned by pretrained VLMs, enabling representations to converge to a minimal sufficient statistic. 
In detail, to accurately estimate the MI between multimodal inputs and representations, we derive a tight upper bound for the symmetrized joint MI, which measures different internal correlations rather than the overall dependency between different modalities. 
More specifically, the upper bound incorporates correlations between single-modal input and representation, as well as the correlation between visual and linguistic representations, guiding VQA models to learn more robust representations and better capture actual relationships. 
In particular, the multimodal representation correlation can facilitate modality alignment.} 
Moreover, to ensure applicability to different transformer architectures, \ie single-stream encoder, two-stream encoder, and encoder-decoder, we unify the internal representations of different pretrained VLMs using the representations after visual and linguistic embedding layers for CIB estimation. 

To demonstrate the proposed CIB, we first provide rigorous theoretical proofs. 
Subsequently, using CIB as the training objective, we finetune pretrained VLMs including VisualBERT \citep{li2019visualbert}, ViLBERT \citep{lu2019vilbert}, VL-BERT$_{\text{B}}$ \citep{su2019vl}, VL-T5 \citep{cho2021unifying}, LXMERT \citep{tan2019lxmert}, UNITER$_{\text{B}}$ \citep{chen2020uniter}, ALBEF \citep{nips_LiSGJXH21}, \RCRevised{mPLUG$_{\text{B}}$} \citep{li2022mplug}, and \RCRevised{BEiT-3$_{\text{B}}$} \citep{wang2022image} \RCRevised{under a standard and clean data setting, and evaluate input robustness on five robustness benchmark datasets:} VQA-Rephrasings \citep{shah2019cycle}, VQA P2 \citep{whitehead2020learning}, IV-VQA \citep{agarwal2020towards}, CV-VQA \citep{agarwal2020towards}, and VQA-CE \citep{dancette2021beyond}. 
Extensive experiments and analyses consistently demonstrate that CIB significantly improves input robustness and exhibits advantages over existing methods when adapting pretrained VLMs to the downstream VQA task. 

\RCRevised{
In summary, our main contributions are as follows: 
(\emph{i}) We propose Correlation Information Bottleneck (CIB), a generic objective that can encourage representations to converge to a minimal sufficient statistic and enhance input robustness when adapting pretrained VLMs to VQA. 
(\emph{ii}) We derive a tight upper bound for the MI between multimodal inputs and representations, incorporating different internal correlations that can guide models to learn more robust representations and facilitate modality alignment. 
(\emph{iii}) Theoretical proofs and extensive experiments evaluate the robustness, superiority, and generalizability of our CIB. 
}

The remainder of the paper is organized as follows: 
Section~\ref{sec:rw} introduces related literature on robustness in VQA, information bottleneck, and vision-language models. 
Section~\ref{sec:method} elaborates on CIB, the application of CIB in adapting pretrained VLMs to VQA, and the theoretical analysis of input robustness for CIB. 
In Section~\ref{sec:exp}, we conduct comprehensive experiments and discussions to demonstrate the effectiveness and superiority of CIB in terms of robustness and accuracy. 
\RCRevised{In Appendix~\ref{sec:app_proof}, we provide a theoretical derivation for CIB and proofs for some proposed theorems.}

\section{Related Work}
\label{sec:rw}

\subsection{Robustness in VQA}  

Recently, in order to promote practical applications, numerous studies have been proposed to investigate various aspects of VQA robustness, such as input robustness \citep{shah2019cycle,whitehead2020learning,agarwal2020towards,kant2020contrast}, human-adversarial robustness \citep{li2021adversarial,sheng2021human}, and robustness against answer distribution shift \citep{agrawal2022rethinking,pan2022causal,kervadec2021roses,jiang2021x,teney2020value,clark2019don,goyal2017making}. 
In this paper, we explore input robustness, which refers to the capability of VQA models to defend against visual and linguistic variations, such as rephrasing questions \citep{shah2019cycle,whitehead2020learning}, manipulating images \citep{agarwal2020towards}, and shortcut learning involved in multimodal inputs \citep{dancette2021beyond}. 
The prevailing method to improve input robustness is data augmentation, \ie generating additional data to train more robust VQA models. 
While data augmentation is a feasible and effective solution, the quality of the generated data is uncontrollable (\eg limited expressiveness and excessive verbosity), and the human-generated process is time-consuming. 
Moreover, cycle-consistency between the original question and its rephrasings \citep{shah2019cycle}, contrastive learning \citep{kant2020contrast}, and adversarial training \citep{li2020closer} have also been introduced to improve input robustness. 
These recent studies demonstrate that state-of-the-art VQA models remain vulnerable to input variation attacks. 
Therefore, in this paper, we focus on further improving the input robustness of existing VQA models.

\subsection{Information Bottleneck}

The Information Bottleneck (IB) principle was originally proposed by Tishby~\etal \citep{tishby2000information} for information compression, and was later applied to analyze deep learning model architectures \citep{tishby2015deep,shwartz2017opening}. 
Essentially, the IB objective is to seek a tradeoff between maximizing predictive accuracy and minimizing representation complexity. 
Some recent research targets exploiting the IB principle to improve model robustness and generalization, especially in domain generalization \citep{du2020learning,li2021invariant}, out-of-distribution generalization \citep{ahuja2021invariance}, multiview representation learning \citep{federici2020learning,bao2021disentangled}, and finetuning of pretrained language models \citep{mahabadi2021variational,wang2021infobert,dong2021should}. 
In addition, some works \citep{wang2022rethinking,zhou2022understanding,pan2020disentangled,jeon2021ib,dubois2020learning} aim to learn disentangled optimal representations from an IB perspective. 
Since IB can facilitate compact and meaningful representation learning, we extend it to multimodal learning and apply IB to obtain robust VQA models. 

\subsection{Vision-Language Models}

Vision-Language pretraining aims to learn task-agnostic visiolinguistic representations for improving the performance of downstream tasks in a finetuning fashion \citep{huang2020pixel,zhou2020unified,shi2020contrastive,li2021semvlp,kim2021vilt,sun2021lightningdot,huang2021seeing,dou2022empirical,zhong2022regionclip,alayrac2022flamingo,xu2023mplug}. 
From the perspective of model architecture, prevailing pretrained vision-language models (VLMs) can be roughly grouped into three types: single-stream encoder \citep{su2019vl,chen2020uniter,gan2020large,li2020oscar,zhang2021vinvl,kim2021vilt}, two-stream encoder \citep{lu2019vilbert,tan2019lxmert,lu202012,yu2020ernie,li2021scheduled}, and encoder-decoder \citep{cho2021unifying,nips_LiSGJXH21,zeng2022multi,li2022blip,wang2022ofa,li2022mplug}. 
Specifically, single-stream models first align image regions and text tokens and then apply a uniform transformer \citep{vaswani2017attention} to learn the contextualized representations. 
Two-stream models first utilize two separate transformers to learn high-level representations for images and texts, and then integrate the two modalities with a cross-modal transformer. 
Encoder-decoder models respectively utilize encoders and decoders to learn multimodal representations and to generate related texts for specific downstream tasks. 
In this paper, we unify the three typical types of VLMs and propose CIB to improve input robustness when adapting these pretrained VLMs for the downstream VQA task.

\section{Methodology} 
\label{sec:method}

In this section, we first present the preliminaries of the problem setting and the general IB principle. 
Then, we elaborate on the proposed CIB in Section~\ref{sec:CIB} and explain how to apply CIB to improve input robustness when adapting pretrained VLMs to VQA in Section~\ref{sec:application_cib}. 

\subsection{Preliminary} 

\noindent\textbf{Problem Setting.} 
In the finetuning process, single-stream and two-stream VLMs usually formulate the VQA task as a multi-answer classification problem, while encoder-decoder VLMs often regard VQA as text generation \citep{cho2021unifying,wang2022simvlm}, \ie generating free-form textual answers for a given question instead of selecting a specific one from the predefined set of answers. 
Given a VQA dataset $\mathcal{D}=\{(I, Q, y)\in \mathcal{I}\times \mathcal{Q}\times \mathcal{Y} \}$, where $I$ is an image, $Q$ is a question, and $y$ is an answer, VLMs take image-question pairs as input, where the image is further represented as a set of image regions or patches $\{v_1, \dots, v_K\}$ ($K$ is the number of regions or patches in one image) and the question is tokenized as a token sequence $\{w_1, \dots, w_L\}$ ($L$ is the number of word tokens in a question). 
For single-stream and two-stream VLMs, they output the answer probability distribution $Y$ using an additional VQA Head module, which is implemented by two fully-connected layers sandwiched with GeLU activation and Layer Normalization operation. 
Meanwhile, encoder-decoder VLMs directly generate textual answers. 

\noindent\textbf{IB View of Representation Learning.} 
From an information-theoretic perspective, seeking a robust representation $T$ in representation learning is equivalent to preserving information about the output $Y$ while removing irrelevant and redundant information from the input $X$. 
This is because for a given task, irrelevant and redundant information may encourage models to learn superfluous correlations between answer labels and inputs. 
Formally, the IB principle \citep{tishby2000information,tishby2015deep} formulates representation learning as an information tradeoff and finds an optimal representation by maximizing the Lagrangian
\begin{equation}
\begin{aligned} 
\mathcal{L}_{\text{IB}} := I(Y; T) - \beta I(X; T),
\end{aligned}
\label{eq:ib} 
\end{equation}
where $\beta \ge 0$ controls the tradeoff between compression and prediction, and $I(\cdot; \cdot)$ denotes mutual information (MI).

\subsection{Correlation Information Bottleneck} 
\label{sec:CIB}

In vision-language representation learning, given two modality inputs $X^v$ and $X^l$, \RCRevised{VLMs learn the corresponding visual and linguistic representations $T^v$ and $T^l$ of some intermediate transformer layers while simultaneously maximizing the MI between the obtained representations and a given label $Y$ to guarantee representations contain sufficient information for predicting $Y$. 
}
To extend the general IB principle to the multimodal setting, we first consider the inputs and internal representations as a whole, \ie $X=[X^v, X^l]$ and $T=[T^v, T^l]$, respectively, and then derive a differentiable estimation for  IB by expanding the MI terms in Eq.~\eqref{eq:ib}. 

Specifically, we first focus on $I(Y; T)$, which can be rewritten using the conditional probability definition: 
\begin{equation}
\begin{aligned}
I(Y; T) = \int p(y, t) \log \frac{p(y|t)}{p(y)} \, dydt \,. 
\end{aligned}
\label{eq:mi}
\end{equation} 
Since the conditional probability $p(y|t)$ is intractable, we instead estimate $I(Y; T)$ with the BA \citep{agakov2004algorithm} lower bound: 
\begin{equation}
\begin{aligned}
I(Y; T) \ge \int p(y, t) \log q(y|t) dy dt - \int p(y) \log p(y) dy , 
\end{aligned}
\label{eq:lower}
\end{equation}
where $q(y|t)$ is an accessible auxiliary distribution for $p(y|t)$ and $- \int p(y) \log p(y) dy = H(Y)$ is the entropy of labels. 
which is independent of the optimization procedure in finetuning. 
Ignoring $H(Y)$, the remaining term of the lower bound in Eq.~\eqref{eq:lower} is equal to $-H(Y|T)$, meaning that maximizing the lower bound of $I(Y; T)$ is equivalent to minimizing the cross-entropy loss of a specific task. 
\RCRevised{
In other words, when using IB as the training objective, maximizing $I(Y; T)$ can be equivalent to minimizing the VQA loss $\mathcal{L}_{\text{vqa}}$. 
}

Next, we consider the mutual information between the input sources and their corresponding representations, that is, the term $I(X; T)$ in Eq.~\eqref{eq:ib}. 
\RCRevised{
To accurately estimate $I(X; T)$, instead of directly measuring the overall dependency between $X$ and $T$ (\ie regarding $X^v$ and $X^l$ as a whole one $X$, and regarding $T^v$ and $T^l$ as a whole one $T$), we consider expanding $I(X; T)$ to $I(X^v, X^l; T^v, T^l)$, and attempt to derive a tight upper bound for it. 
Since $I(X^v, X^l; T^v, T^l)$ incorporates different internal correlations, such as the correlation between visual input $X^v$ and representation $T^v$, the correlation between linguistic input $X^l$ and representation $T^l$, and the correlation between visual and linguistic representations ($T^v$ and $T^l$). 
These correlations may guide models to learn more compact visual and linguistic representations and facilitate modality alignment between visual and linguistic representations. 
}
Therefore, we propose to maximize the Correlation Information Bottleneck (CIB) formula: 
\begin{equation}
\begin{aligned} 
\mathcal{L}_{\text{CIB}} := I(Y; T) - \beta I(X^v, X^l; T^v, T^l), 
\end{aligned}
\label{eq:CIB}
\end{equation}
where $I(X^v, X^l; T^v, T^l)$ is a symmetrized variant of joint mutual information \citep{bennasar2015feature} that considers the internal correlations between $X=[X^v, X^l]$ and $T=[T^v, T^l]$. 
To efficiently estimate $I(X^v, X^l; T^v, T^l)$, we first further expand it conditioned on the properties of mutual information and the data processing inequality in representation learning \citep{federici2020learning}. 
The derivation can be formally stated by Theorem~\ref{thm:JMI} (\cf Section~\ref{sec:app_proof_theorem1} for proof): 

\begin{theorem} (Upper Bound of $I(X^v, X^l; T^v, T^l)$) 
Given two groups of random variables $X=[X^v, X^l]$ and $T=[T^v, T^l]$, the MI $I(X^v, X^l; T^v, T^l)$ can be upper-bounded with 
\begin{align}
I(X; T) &= I(X^v, X^l; T^v, T^l), \notag \\
&\le I(X^v; T^v) + I(X^l; T^l) \textcolor{magenta}{- I(T^v; T^l)} + \SKL,
\label{eq:JMI}
\end{align}
where $\SKL$ denotes the symmetric Kullback-Leibler (KL) divergence that can be calculated by averaging the divergences $\KL(p(t^v|x^v)||p(t^l|x^l))$ and $\KL(p(t^l|x^l)||p(t^v|x^v))$. 
\label{thm:JMI}
\end{theorem}

After approximating the MI $I(X^v, X^l; T^v, T^l)$, the lower bound of $\mathcal{L}_{\text{CIB}}$ can be stated as the following Theorem~\ref{thm:CIB}. 

\begin{theorem}(Lower Bound of CIB) 
Given random variable $X = [X^v, X^l]$, 
two deterministic functions $f_{\theta^v}$ and $f_{\theta^l}$ let $T^v = f_{\theta^v}(X^v)$ and $T^l = f_{\theta^l}(X^l)$. 
Correlation Information Bottleneck (CIB) can then be bounded as 
\begin{align}
&\mathcal{L}_{\text{CIB}} = I(Y; T) - \beta I(X^v, X^l; T^v, T^l), \notag \\ 
&\ge I(Y; T) - \beta\Big[I(X^v; T^v) + I(X^l; T^l) \textcolor{magenta}{- I(T^v; T^l)}  + \SKL \Big]. 
\label{eq:CIB_all}
\end{align}
\label{thm:CIB}
\vspace{-7mm}
\end{theorem}

In summary, Theorem~\ref{thm:CIB} suggests that in vision-language representation learning, if $I(Y; T)$ is considered a task-related objective, $I(X^v, X^l; T^v, T^l)$ can be viewed as a regularizer used to constrain the compactness and redundancy of the learned representations. 
Overall, CIB encourages pretrained VLMs to learn more robust representations by seeking an optimal tradeoff between redundancy and compression in representations. 
\RCRevised{Moreover, CIB facilitates modality alignment and correlation by maximizing the MI $I(T^v; T^l)$ between visual and linguistic representations.}

\subsection{Adapting Pretrained VLMs to VQA with CIB} 
\label{sec:application_cib}

\begin{figure*}[!t]
\centering 
\subfigure[Single-Stream Encoder]{\label{fig:ss}\includegraphics[width=0.28\textwidth]{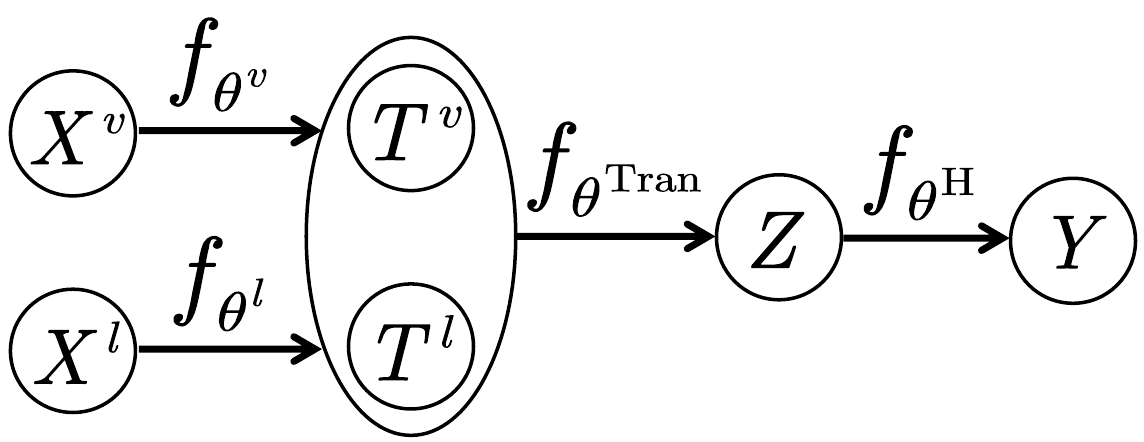}
}
\quad
\subfigure[Two-Stream Encoder]{
\label{fig:ts}
\includegraphics[width=0.36\textwidth]{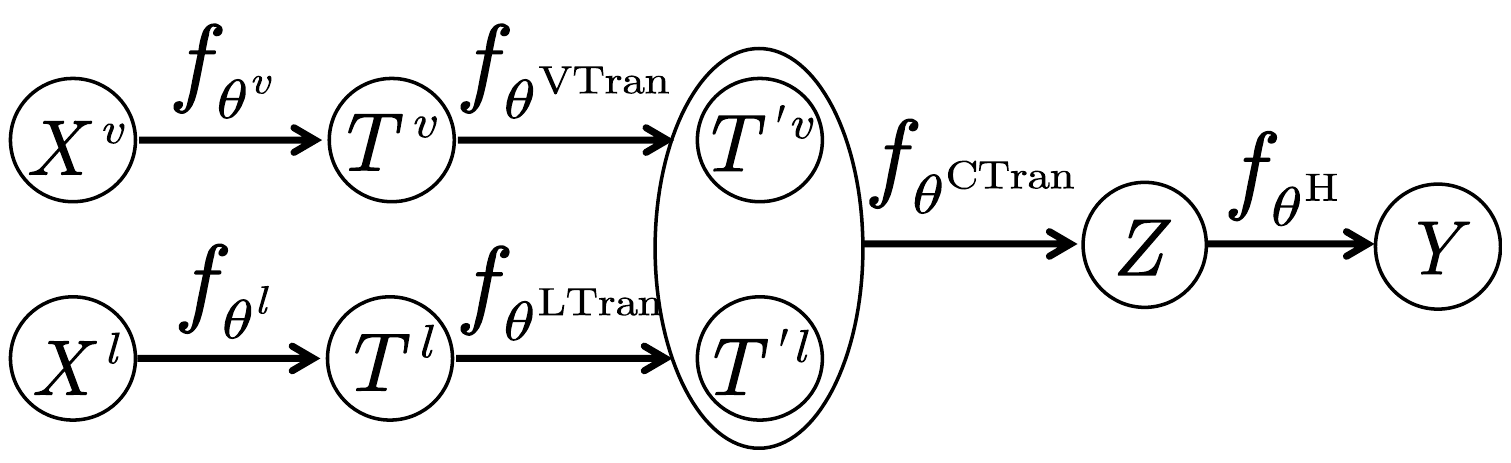}
}
\quad
\subfigure[Encoder-Decoder]{\label{fig:ed}
\includegraphics[width=0.28\textwidth]{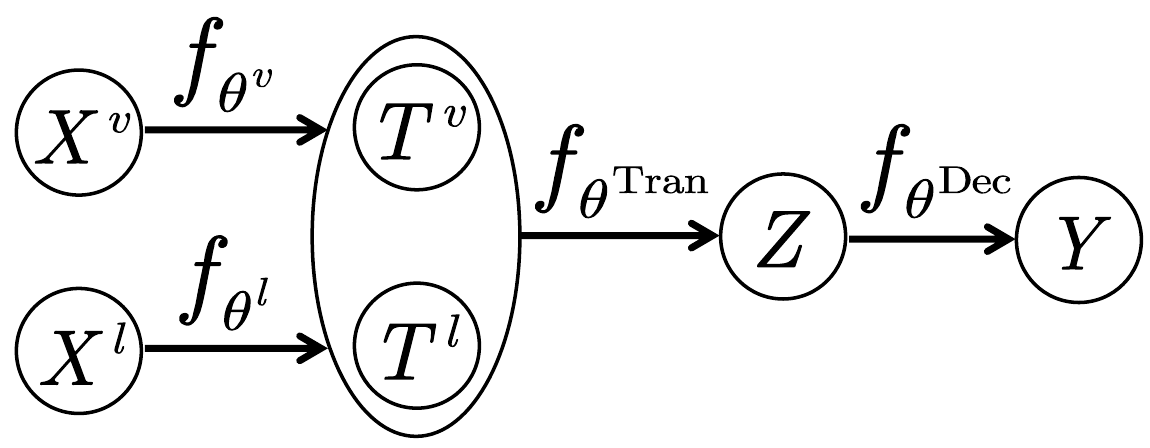}
}
\vspace{-2mm}
\caption{The information flow of three typical transformer architectures of VLMs}
\label{fig:architecture}
\vspace{-2mm}
\end{figure*}

\begin{figure}[!t]
\centering
\includegraphics[width=0.9\linewidth]{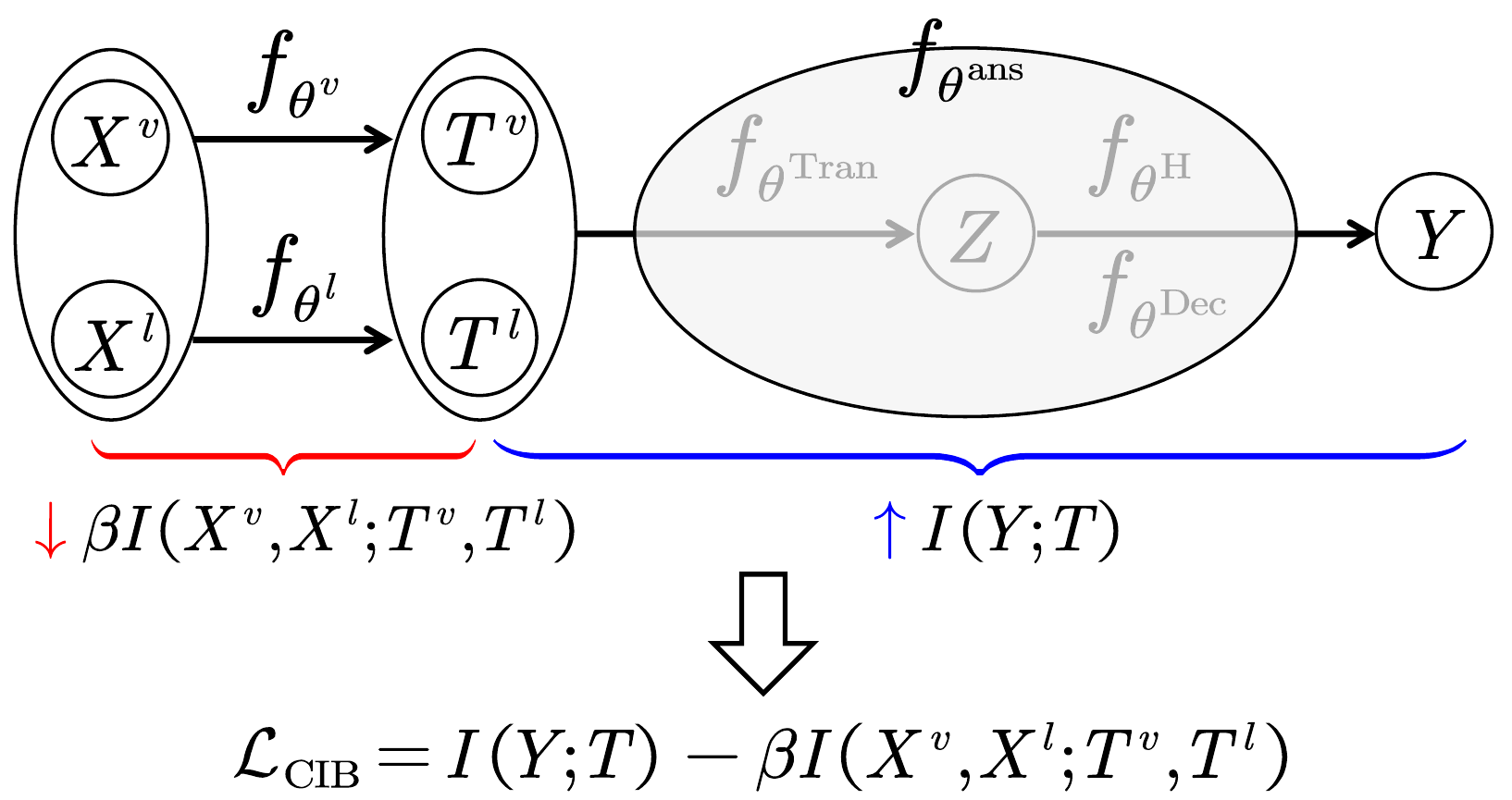}
\caption{Illustration of using CIB to adapt pretrained VLMs to downstream task. 
CIB seeks a minimal sufficient statistic by minimizing MI between input and internal representation ($\textcolor{red}{\downarrow}$) while maximizing MI between output and representation ($\textcolor{blue}{\uparrow}$) 
}
\label{fig:unify_cib}
\vspace{-2mm}
\end{figure}

As illustrated in \figref{ss}, \ref{fig:ts}, and \ref{fig:ed}, there are three typical transformer architectures for VLMs: single-stream encoder \citep{li2019visualbert,su2019vl,chen2020uniter}, two-stream encoder \citep{tan2019lxmert,lu2019vilbert}, and encoder-decoder \citep{cho2021unifying,nips_LiSGJXH21}. 
When finetuning pretrained VLMs with CIB, to unify the three architectures into a single formulation, as shown in \figref{unify_cib}, we utilize the region-level or patch-level visual features after the visual Embedding layer (\ie $f_{\theta^v}$ is the parametric embedding layer) as the internal visual representation $T^v$. 
Analogously, the token-level linguistic features after the linguistic embedding layer ($f_{\theta^l}$) are considered as the internal linguistic representation $T^l$. 
All subsequent Transformer layers ($f_{\theta^{{\text{Tran}}}}$) and the VQA Head module ($f_{\theta^\text{H}}$) for the single-stream and two-stream VLMs as well as the Decoder ($f_{\theta^\text{Dec}}$) for the encoder-decoder VLMs serve as the parametric approximator \RCRevised{($f_{\theta^{\text{ans}}}$)} to generate $Y$ given $T=[T^v, T^l]$. 
\RCRevised{
As summarized in Algorithm~\ref{alg:algorithm1}, we first convert $I(Y; T)$ to the cross-entropy loss ($\mathcal{L}_{\text{vqa}}$) for answer prediction in VQA and estimate the remaining terms in Theorem~\ref{thm:CIB}. 
After obtaining $\mathcal{L}_{\text{CIB}}$, we update all parameters by minimizing $- \mathcal{L}_{\text{CIB}}$. 
Next, we elaborate on the estimation of CIB terms. 
}

\subsubsection{Estimating CIB Terms} 

As stated in Theorem~\ref{thm:CIB}, in addition to the task-related MI term $I(Y;T)$, $I(X^v, X^l; T^v, T^l)$ can be further decomposed into four computable MI terms. 
\RCRevised{
Firstly, we focus on the MI between inputs and representations within a single visual or linguistic modality. 
The inputs $X^v$ and $X^l$ are intrinsically two sets of random variables, \ie $X^v = [X^v_{1}, ..., X^v_{K}]$ and $X^l = [X^l_{1}, ..., X^l_{L}]$. 
The functions $f_{\theta^v}$ and $f_{\theta^l}$ transform $X^v$ and $X^l$ into visual and linguistic representations, respectively, such that $T^v = [T^v_{1}, ..., T^v_{K}] = [f_{\theta^v}(X^v_{1}), ..., f_{\theta^v}(X^v_{K})]$ and $T^l = [T^l_{1}, ..., T^l_{L}] = [f_{\theta^l}(X^l_{1}), ..., f_{\theta^l}(X^l_{L})]$. 
While for sample pairs $\{(X_i^v, T_i^v)\}_{i=1}^{K}$ and $\{(X_i^l, T_i^l)\}_{i=1}^{L}$, the conditional probability distributions $p(t^v|x^v)$ and $p(t^l|x^l)$ are known during the finetuning process. 
}
Consequently, we adopt a sample-based differentiable MI estimator, CLUB \citep{cheng2020club}, to approximate the upper bound of the MI between visual or linguistic inputs and their corresponding representations, \ie 
\begin{align} 
&\hat{I}(X^v; T^v) = \frac{1}{K^2} \sum_{i=1}^{K}\sum_{j=1}^{K} \Big[\log p(t_i^v|x_i^v) - \log p(t_j^v|x_i^v)\Big], \\ 
&\hat{I}(X^l; T^l) = \frac{1}{L^2} \sum_{i=1}^{L}\sum_{j=1}^{L} \Big[\log p(t_i^l|x_i^l) - \log p(t_j^l|x_i^l)\Big]. 
\label{eq:IB_sum_term}
\end{align}

\begin{algorithm}[!t]
\RCRevised{
\caption{\small\RCRevised{Finetuning pretrained VLMs with CIB for VQA}}
\label{alg:algorithm1}
\algorithmicrequire{
\,\, Visual (Image) sequence: $X^v$; Hyperparameter: $\beta$; \newline  
\textcolor{white}{\textbf{Input:\quad}} Linguistic (Question) sequence: $X^l$.
\newline  
}
\algorithmicensure{ 
~Training loss: $- \mathcal{L}_{\text{CIB}}$. 
}
\begin{algorithmic}[1]
\State Load pretrained weights for VLMs; 
\State $T^v \leftarrow$ $f_{\theta^v}(X^v)$, $T^l \gets$ $f_{\theta^l}(X^l)$; 
\State $Y \gets$ $f_{\theta^{\text{ans}}}([T^v, T^l])$; 
\Procedure{Estimate $\mathcal{L}_{\text{CIB}}$}{} 

\State \textcolor{gray}{\textit{\# estimate all terms in $\mathcal{L}_{\text{CIB}}$}}
\State $\mathcal{L}_{\text{vqa}} \gets$ convert $I(Y; T)$ to VQA loss;

\State $\hat{I}(X^v; T^v) \gets$ estimate MI of $X^v$ and $T^v$ with CLUB; 
\State $\hat{I}(X^l; T^l) \gets$ estimate MI of $X^l$ and $T^l$ with CLUB; 

\State $\hat{I}(\bar{T}^v; \bar{T}^l) \gets$ estimate MI of $T^v$ and $T^l$ using NWJ with $f_{\theta_{\text{FC}}}$; 
\State $\hat{D}_{\text{skl}} \leftarrow$ estimate symmetric KL between $p(t^l|x^l)$ and $p(t^v|x^v)$; 

\State \textcolor{gray}{\textit{\# compute $\mathcal{L}_{\text{CIB}}$}}
\State $\mathcal{L}_{\text{CIB}} = -\mathcal{L}_{\text{vqa}} -\beta\left[\hat{I}(X^v; T^v) + \hat{I}(X^l; T^l) - \hat{I}(\bar{T}^v; \bar{T}^l) + \hat{D}_{\text{skl}}\right]$; 

\EndProcedure
\State Update $f_{\theta^v}$, $f_{\theta^l}$, $f_{\theta^{\text{ans}}}$, $f_{\theta_{\text{FC}}}$ by minimizing $- \mathcal{L}_{\text{CIB}}$. 
\end{algorithmic}
}
\end{algorithm}

For \textcolor{magenta}{$I(T^v; T^l)$}, it is challenging to estimate directly due to the different sequence lengths of $T^v \in \mathbb{R}^{K\times d}$ and $T^l\in \mathbb{R}^{L\times d}$. 
Therefore, we transform the two sequence representations into the global visual and linguistic representations $\bar{T}^v \in \mathbb{R}^{d}$ and $\bar{T}^l\in \mathbb{R}^{d}$, using a one-layer fully-connected (FC) network. 
To guarantee that the inequality in Eq.~\eqref{eq:CIB_all} holds, we should approximate the lower bound of $I(T^v; T^l)$. 
Therefore, we estimate $I(T^v; T^l)$ with NWJ \citep{poole2019variational}, \ie 
\begin{align} 
&\hat{I}(\bar{T}^v; \bar{T}^l) = \notag \\ 
&\mathbb{E}_{p(\bar{t}^v, \bar{t}^l)}\left[\log f_{\theta_{\text{FC}}}(\bar{t}^v, \bar{t}^l)\right] - \frac{1}{e}\mathbb{E}_{p(\bar{t}^v) p(\bar{t}^l)} \left[f_{\theta_{\text{FC}}}(\bar{t}^v, \bar{t}^l)\right],
\label{eq:IB_sum_term_rep}
\end{align} 
where $f_{\theta_{\text{FC}}}$ denotes the discriminant function implemented using a two-layer FC network.

Finally, for $\SKL$ in $\mathcal{L}_{\text{CIB}}$, since $p(t^l|x^l)$ and $p(t^v|x^v)$ have a known probability density, we can directly compute the two KL divergences using internal visual and linguistic representations. 
That is, $\SKL$ can be obtained by 
\begin{align} 
&\hat{D}_{\text{skl}} = \notag \\
&\frac{1}{2} \left[\KL\left(p(t^v|x^v)||p(t^l|x^l)\right) + \KL\left(p(t^l|x^l)||p(t^v|x^v)\right)\right]. 
\label{eq:dskl}
\end{align}

\begin{table*}[!t]
\begin{center}
\caption{Details on input robustness datasets}
\label{tab:benchmark}
\vspace{-1mm}
\setlength{\tabcolsep}{2.1mm}{
\begin{tabularx}{\linewidth}{@{}lcccccccc@{}}
\toprule
\multirow{2}{*}{Datasets} 
&\multirow{2}{*}{Perturbation} 
&\multirow{2}{*}{Metric}
&\multirow{2}{*}{QType} 
&\multicolumn{1}{c}{\multirow{2}{*}{\makecell{Finetuning\\Dataset}}}
&\multicolumn{4}{c}{Evaluation} 
\\ 
\cmidrule(l){6-9}
& & &
&
&len(Q) &\#IQ &\#PER/CE &\#ORI/Easy 
\\ 
\midrule[1pt] 
VQA-Rephrasings \citep{shah2019cycle} 
&Rephrasing 
&CS($m$) 
&All 
&VQA v2 train
&7.15 
&\eqmakebox[log][r]{162k} 
&\eqmakebox[log][r]{121,516} 
&\eqmakebox[log][r]{40,504} 
\\ 
VQA P2 \citep{whitehead2020learning} 
&Par\&Syn\&Ant 
&CS($m$) 
&All 
&VQA v2 train
&6.32 
&\eqmakebox[log][r]{52k} 
&\eqmakebox[log][r]{26,512} 
&\eqmakebox[log][r]{25,814}
\\ 
IV-VQA \citep{agarwal2020towards} 
&Invariant object 
&\#flips 
&All 
&VQA v2 train
&5.85 
&\eqmakebox[log][r]{120k} 
&\eqmakebox[log][r]{83,700} 
&\eqmakebox[log][r]{36,181} 
\\ 
CV-VQA \citep{agarwal2020towards} 
&Covariant object 
&\#flips 
&Num  
&VQA v2 train
&5.83 
&\eqmakebox[log][r]{4k} 
&\eqmakebox[log][r]{4,141} 
&\eqmakebox[log][r]{2,641}
\\ 
VQA-CE \citep{dancette2021beyond} 
&Counterexample 
&- 
&All 
&VQA v2 train
&6.19 
&\eqmakebox[log][r]{214k} 
&\eqmakebox[log][r]{63,298} 
&\eqmakebox[log][r]{147,681}
\\ 
\bottomrule
\end{tabularx}
}
\end{center}
\vspace{-4mm}
\end{table*}

\RCRevised{
\subsubsection{Theoretical Justification for Input Robustness}

In the following section, we conduct perform a theoretical analysis of input robustness for CIB. 
Formally, for a perturbation $\delta$ added to visual and linguistic inputs, let $X'=[{X^v}', {X^l}']$ represent the perturbed inputs of standard inputs $X=[X^v, X^l]$, \ie $X' = X + \delta$. 
Functions $f_{\theta^v}$ and $f_{\theta^l}$ transform $X=[X^v, X^l]$ and $X'=[{X^v}', {X^l}']$ into $T = [T^v, T^l] = [f_{\theta^v}(X^v), f_{\theta^l}(X^l)]$ and $T' = [{T^v}', {T^l}'] = [f_{\theta^v}({X^v}'), f_{\theta^l}({X^l}')]$, with $T\neq T'$. 
The distributions of $X$ and $X'$ are denoted by probabilities $p(x)$ and $q(x)$, where $q(x)$ approximates the distribution of $p(x)$. 
$\delta_m$ is the maximum perturbation bound that does not alter the output label, \ie $Y= f_{\theta^{\text{ans}}}(T)=f_{\theta^{\text{ans}}}(T')$ when $||\delta|| \le \delta_m$. 
According to the definition of CIB, the performance gap between standard inputs and perturbed inputs is $|I(T; Y) - I(T'; Y)|=|I(T^v, T^l; Y) - I({T^v}', {T^l}'; Y)|$. 
To provide theoretical justification for the performance gap, based on the work \citep{wang2021infobert}, we derive the upper bound 
\begin{align}
&|I(T; Y) - I(T'; Y)| \notag \\
&= |I(T^v, T^l; Y) - I({T^v}', {T^l}'; Y)|, \notag \\ 
&\leq B^v_1 \sqrt{\mathcal{T}^v} \left(I(X^v;T^v)\right)^{1/2} + 
B^v_2 |\mathcal{T}^v|^{3/4} \left(I(X^v;T^v)\right)^{1/4} \notag \\
&~~~ + B^v_3 \sqrt{|\mathcal{T}^v|} \left(I({X^v}';{T^v}')\right)^{1/2} + 
B^v_4 |\mathcal{T}^v|^{3/4} \left(I({X^v}';{T^v}')\right)^{1/4} \notag \\
&~~~ + B^l_1 \sqrt{\mathcal{T}^l} \left(I(X^l;T^l)\right)^{1/2} + 
B^l_2 |\mathcal{T}^l|^{3/4} \left(I(X^l;T^l)\right)^{1/4} \notag \\
&~~~ + B^l_3 \sqrt{|\mathcal{T}^l|} \left(I({X^l}';{T^l}')\right)^{1/2} + 
B^l_4 |\mathcal{T}^l|^{3/4} \left(I({X^l}';{T^l}')\right)^{1/4} \notag \\
&~~~ + B^v_0 + B^l_0, 
\label{eq:justification}
\end{align}
where $\mathcal{T}^v$ is the finite support of $T^v$ and ${T^v}'$, and $B^v_0$, $B^v_1$, $B^v_2$, $B^v_3$, and $B^v_4$ are constants that depend on the sequence length $K$, $\delta$, and $p(x^v)$. 
$\mathcal{T}^l$ is the finite support of $T^l$ and ${T^l}'$, and $B^l_0$, $B^l_1$, $B^l_2$, $B^l_3$, and $B^l_4$ are constants that depend on the sequence length $L$, $\delta$, and $p(x^l)$ (\cf Section~\ref{sec:proof_performance_bound} for proof). 
}

\begin{table*}[!h]
\begin{center}
\caption{Summary of baseline pretrained VLMs (\RCRevised{AC: Answer Classification, TG: Text Generation})}
\label{tab:summary_VLP} 
\vspace{-1mm}
\setlength{\tabcolsep}{3.7mm}{
\begin{tabularx}{\linewidth}{@{}lllllc}
\toprule
Domain
&Pretrained VLMs
&Architecture
&Pretraining IT Datasets
&Image Tokens 
&\RCRevised{VQA}
\\ 
\midrule[1pt]
\multirow{1}{*}{ID}
&VisualBERT \citep{li2019visualbert} 
&\multirow{1}{*}{\makecell{single-stream}}
&\multirow{1}{*}{COCO} 
&\multirow{1}{*}{Region Feature}  
&\RCRevised{AC}
\\ 
\midrule
&\multirow{1}{*}{\makecell{ViLBERT \citep{lu2019vilbert}}}
&\multirow{1}{*}{\makecell{two-stream}}
&\multirow{1}{*}{CC}
&\multirow{1}{*}{Region Feature}  
&\RCRevised{AC}
\\ 
\multirow{-2}{*}{OOD}
&\multirow{1}{*}{\makecell{VL-BERT$_{\text{B}}$ \citep{su2019vl}}} 
&\multirow{1}{*}{\makecell{single-stream}} 
&\multirow{1}{*}{CC} 
&\multirow{1}{*}{Region Feature} 
&\RCRevised{AC}
\\ 
\midrule
\multirow{6}{*}{ID+OOD}
&VL-T5 \citep{cho2021unifying} 
&\multirow{1}{*}{\makecell{encoder-decoder}}
&\multirow{1}{*}{\makecell{COCO,VG,GQA,VQA,VGQA}} 
&\multirow{1}{*}{Region Feature} 
&\RCRevised{TG}
\\
&\multirow{1}{*}{\makecell{LXMERT \citep{tan2019lxmert}}}
&\multirow{1}{*}{\makecell{two-stream}}
&\multirow{1}{*}{\makecell[c]{COCO,VG,GQA,VQA,VGQA}}
&\multirow{1}{*}{Region Feature}
&\RCRevised{AC}
\\ 
&\multirow{1}{*}{\makecell{UNITER$_{\text{B}}$ \citep{chen2020uniter}}}
&\multirow{1}{*}{\makecell{single-stream}}
&\multirow{1}{*}{\makecell[c]{COCO,VG,SUB,CC}}
&\multirow{1}{*}{Region Feature} 
&\RCRevised{AC}
\\ 
&\multirow{1}{*}{\makecell{ALBEF \citep{nips_LiSGJXH21}}}
&\multirow{1}{*}{encoder-decoder} 
&\multirow{1}{*}{COCO,VG,SUB,CC}
&\multirow{1}{*}{Patch Embedding}
&\RCRevised{TG}
\\ 
&\multirow{1}{*}{\makecell{\RCRevised{mPLUG$_{\text{B}}$} \citep{li2022mplug}}}
&\multirow{1}{*}{\RCRevised{encoder-decoder}}
&\multirow{1}{*}{\RCRevised{COCO,VG,SBU,CC,CC12M}}
&\multirow{1}{*}{\RCRevised{Patch Embedding}}  
&\RCRevised{TG}
\\ 
&\multirow{1}{*}{\makecell{\RCRevised{BEiT-3$_{\text{B}}$} \citep{wang2022image}}}
&\multirow{1}{*}{\makecell{\RCRevised{two-stream}}}
&\multirow{1}{*}{\RCRevised{COCO,VG,SBU,CC,CC12M}}
&\multirow{1}{*}{\RCRevised{Patch Embedding}}  
&\RCRevised{AC}
\\ 
\bottomrule
\end{tabularx}
}
\end{center}
\vspace{-4mm}
\end{table*}

\section{Experiment}
\label{sec:exp}

\RCRevised{
In this section, we evaluate the input robustness of the proposed CIB and carry out detailed ablation studies to analyze the performance contribution of CIB components. 
Meanwhile, we explore the effectiveness of CIB in some other cases, such as standard VQA performance, adversarial attacks, and other multimodal tasks beyond VQA. 
}

\subsection{Experimental Settings}

\subsubsection{Evaluation Datasets}
\RCRevised{Unless otherwise specified, we finetune pretrained VLMs on the standard and clean VQA v2 training set \citep{goyal2017making} and evaluate input robustness on five robustness benchmark datasets}: VQA-Rephrasings \citep{shah2019cycle}, VQA P2 \citep{whitehead2020learning}, IV-VQA \citep{agarwal2020towards}, CV-VQA \citep{agarwal2020towards}, and VQA-CE \citep{dancette2021beyond}. 
VQA-Rephrasings and VQA P2 evaluate robustness against linguistic variations, while IV-VQA and CV-VQA evaluate robustness against visual variations. 
VQA-CE, on the other hand, assesses robustness against shortcut learning involving inputs. 
As all these datasets are built on the VQA v2 \citep{goyal2017making} validation split, \textit{we consequently train our models only on the VQA v2 training set}. 

Table~\ref{tab:benchmark} summarizes dataset details, including the type of perturbation, specific evaluation metrics for robustness, question type (QType), shared dataset for finetuning, and the test datasets statistics. 
These statistics encompass the total number of image-question pairs (\#IQ), perturbation samples (\#PER/CE), original samples (\#ORI/Easy), and the average question length (len(Q)). 
Specifically, VQA-Rephrasings averagely collects 3 rephrasings for each of the 40,504 questions sampled from the VQA v2 validation set, resulting in approximately 162k image-question pairs. 
VQA P2 creates three types of linguistic perturbations, \ie sentence rephrasing (Par), and word substitution with synonyms (Syn) or antonyms (Ant), for 25,814 sampled questions, ultimately obtaining roughly about 52k image-question pairs. 
IV-VQA employs a GAN-based resynthesis technique to remove objects irrelevant to the given question from the image, such that object removal does not affect the answer. 
Conversely, CV-VQA focuses on counting questions (Num) and removes one relevant object, causing the predicted answer on the number of objects to be reduced by one. 
In total, IV-VQA and CV-VQA contain approximately 120k and 4k image-question pairs, respectively. 
VQA-CE is an evaluation benchmark for multimodal shortcuts involved in images and questions. 
It utilizes the detected shortcuts from the training set to obtain 63,298 counterexamples, where all shortcuts lead to incorrect answers, from the VQA v2 validation set. 
Additionally, VQA-CE constructs 147,681 easy examples in which at least one shortcut provides the correct answer. 

\RCRevised{Moreover, to analyze the effectiveness of CIB on standard VQA performance, we conduct experiments on VQA v2 \citep{goyal2017making}. 
Specifically, we first utilize CIB as the training objective to finetune pretrained VLMs on the VQA v2 training and validation sets, and subsequently test standard VQA performance on VQA v2 test-dev. 
To evaluate the generalizability of CIB to other multimodal tasks, we perform experiments on RefCOCO+ \citep{yu2016modeling} in weakly-supervised setups. 
This dataset contains a total of 141,564 expressions based on images from the COCO training set. 
To assess the effectiveness of CIB in addressing human-adversarial attacks, we evaluate our method on AdVQA \citep{sheng2021human}, a human-adversarial benchmark built upon VQA v2 images, featuring approximately 10k/36.8k image-question pairs for the validation/test split. 
}

\subsubsection{Evaluation Metrics}
We follow previous work \citep{antol2015vqa} to evaluate the VQA performance of our methods with VQA-Score. 
In addition, we evaluate robustness against linguistic variations using Consensus Score (CS) \citep{shah2019cycle}, which is the ratio of the number of subsets where all questions are answered correctly to the total number of subsets of size $m$. 
Specifically, for each question group $Q$ containing one original question and its $n$ corresponding rephrasings, all subsets of size $m$ amount to $^{n}C_{m}$, CS can then be defined as 
\begin{equation}
\begin{aligned}
\mathrm{CS}(m) = \sum\nolimits_{q\in Q^{\prime}, Q^{\prime} \subset Q, |Q^{\prime}|=m} \frac{\InF_{Q^{\prime}}(q)}{^{n}C_{m}},  
\end{aligned}
\end{equation}
where $\InF$ is an indicator function defined on $Q^{\prime}$ and $\InF_{Q^{\prime}}(q)$ means a set where the answer to question $q$ is correct. 
Naturally, the higher the average CS at larger values of $m$, the more robust the model. 
To evaluate robustness against visual variations, we utilize \#flips \citep{agarwal2020towards} as a robustness evaluation metric. 
\#flips represents the ratio of the number of prediction mismatches before and after visual content manipulation to the total number of all samples. 
In IV-VQA, if the predicted answers for the original image and the corresponding edited image differ, the prediction is deemed ``flipped''. 
In CV-VQA, an answer to a question based on an edited image is considered to be ``flipped'' if it is not one less than the prediction on the original image.

\begin{table}[!t]
\begin{center}
\small
\caption{Configuration setups}
\label{tab:config}
\vspace{-1mm}
\setlength{\tabcolsep}{2.55mm}{
\begin{tabularx}{\linewidth}{@{}llllll@{}}
\toprule
\multirow{1}{*}{Methods} 
&$\beta$
&$K$
&$L$
&peak lr
&bs
\\ 
\midrule[1pt]
VisualBERT + CIB 
&$5\times10^{-5}$ 
&100 
&20 
&$2\times10^{-5}$ 
&64
\\ 
ViLBERT + CIB
&$1\times10^{-4}$ 
&10-100 
&20 
&$4\times10^{-5}$
&64
\\ 
VL-BERT$_{\text{B}}$ + CIB
&$1\times10^{-4}$ 
&10-100 
&20 
&$4\times10^{-5}$ 
&64
\\ 
VL-T5 + CIB
&$1\times10^{-4}$ 
&36 
&20  
&$4\times10^{-5}$  
&64
\\ 
LXMERT + CIB 
&$5\times10^{-5}$ 
&36 
&20 
&$2\times10^{-5}$ 
&64
\\ 
UNITER$_{\text{B}}$ + CIB 
&$1\times10^{-4}$ 
&10-100 
&20 
&$4\times10^{-5}$ 
&64
\\ 
ALBEF + CIB 
&$1\times10^{-4}$ 
&900 
&30 
&$2\times10^{-5}$ 
&32
\\ 
mPLUG$_{\text{B}}$ + CIB
&$1\times10^{-4}$ 
&900 
&80 
&$2\times10^{-5}$ 
&16
\\ 
BEiT-3$_{\text{B}}$ + CIB 
&$1\times10^{-4}$ 
&900 
&64 
&$2\times10^{-5}$ 
&16
\\ 
\bottomrule
\end{tabularx}
}
\end{center}
\vspace{-4mm}
\end{table}

\begin{table*}[!t]
\begin{center}
\caption{Results of robustness against linguistic variations (\ie sentence rephrasing) on the VQA-Rephrasings dataset \citep{shah2019cycle}
}
\label{tab:all_on_vqa_rep}
\vspace{-1mm}
\setlength{\tabcolsep}{2.7mm}{
\begin{tabularx}{\linewidth}{@{}lllllll}
\toprule
\multirow{2}{*}{Methods}
&\multicolumn{2}{c}{VQA-Score} 
&\multicolumn{4}{c}{Robustness Metric} 
\\
\cmidrule(lr){2-3}
\cmidrule(lr){4-7}
&PER
&ORI
&CS(1)
&CS(2)
&CS(3)
&CS(4)
\\ 
\midrule[1.pt]
\textit{Data Augmentation}
\\
BUTD \citep{anderson2018bottom} 
&51.22 &61.51 
&60.55 &46.96 &40.54 &34.47 
\\ 
\quad + CC \citep{shah2019cycle}
&52.58 (\positive{+1.36}) &62.44 (\positive{+0.93})
&61.66 (\positive{+1.11}) &50.79 (\positive{+3.83}) 
&44.68 (\positive{+4.14}) &42.55 (\positive{+8.08}) 
\\ 
Pythia \citep{jiang2018pythia} 
&54.20 &64.08 
&63.43 &52.03 &45.94 &39.49 
\\ 
\quad + CC \citep{shah2019cycle} 
&55.65 (\positive{+1.45}) &{64.52} (\positive{+0.44}) 
&{64.36} (\positive{+0.93}) &{55.45} (\positive{+3.42}) 
&{50.92} (\positive{+4.98}) &{44.30} (\positive{+4.81})	
\\
BAN \citep{kim2018bilinear} 
&55.87 &64.97 
&64.88 &53.08 
&47.45 &39.87	
\\
\quad + CC \citep{shah2019cycle} 
&56.59 (\positive{+0.72}) &65.87 (\positive{+0.90}) 
&65.77 (\positive{+0.89}) &56.94 (\positive{+3.86}) 
&51.76 (\positive{+4.31}) &48.18 (\positive{+8.31})	
\\ 
MMT \citep{kant2020contrast} 
&- &- 
&67.58 &60.04 &55.53 &52.36 
\\
\quad ConClaT \citep{kant2020contrast} 
&- &- 
&{68.62} (\positive{+1.04}) &{61.42} (\positive{+1.38}) 
&{57.08} (\positive{+1.55}) &{53.99} (\positive{+1.63}) 
\\
\midrule
\textit{w/o Data Augmentation}
\\
UNITER$_{\text{B}}$ \citep{chen2020uniter}
&- &- &71.29 &63.95 &59.48 &56.31 
\\ 
\quad MANGO$_{\text{B}}$ \citep{li2020closer} 
&- &- 
&72.66 (\positive{+1.37}) &66.03 (\positive{+2.08}) 
&61.92 (\positive{+2.44}) &58.95 (\positive{+2.64}) 
\\
VILLA$_{\text{B}}$ \citep{gan2020large}
&- &- 
&72.18 &65.28 
&60.99 &57.93
\\
\quad MANGO$_{\text{VB}}$ \citep{li2020closer} 
&- &- 
&72.78 (\positive{+0.60}) &65.97 (\positive{+0.69}) 
&61.70 (\positive{+0.71}) &58.59 (\positive{+0.66})
\\
\cmidrule(){2-7}
VisualBERT \citep{li2019visualbert}\textcolor{red}{$^\dagger$}
&62.03 &68.46
&70.44 &62.84 &58.41 &55.06
\\
\quad + CIB 
&63.10 (\positive{+1.07}) &69.78 (\positive{+1.32})
&71.85 (\positive{+1.41}) &64.16 (\positive{+1.32}) 
&59.54 (\positive{+1.13}) &56.31 (\positive{+1.25})
\\ 
ViLBERT \citep{lu2019vilbert}\textcolor{red}{$^\dagger$}
&59.16 &67.65
&68.00 &59.65 
&54.68 &51.22
\\
\quad + CIB
&62.28 (\positive{+3.12}) &69.15 (\positive{+1.50}) 
&71.05 (\positive{+3.05}) &63.54 (\positive{+3.89}) 
&59.04 (\positive{+4.36}) &55.89 (\positive{+4.67})
\\ 
VL-BERT$_{\text{B}}$ \citep{su2019vl}\textcolor{red}{$^\dagger$}
&59.89 &67.14
&67.95 &60.11 &55.34 &52.99
\\
\quad + CIB
&60.86 (\positive{+0.97}) &68.74 (\positive{+1.60}) 
&70.52 (\positive{+2.57}) &63.46 (\positive{+3.35}) 
&58.75 (\positive{+3.41}) &53.89 (\positive{+0.90})
\\ 
VL-T5 \citep{cho2021unifying}\textcolor{red}{$^\dagger$}
&65.64 &-
&71.78 &65.35 &62.68 &61.00
\\
\quad + CIB
&66.93 (\positive{+1.29}) &-
&73.65 (\positive{+1.87}) &67.48 (\positive{+2.13}) 
&64.48 (\positive{+1.80}) &62.53 (\positive{+1.53})
\\ 
LXMERT \citep{tan2019lxmert}\textcolor{red}{$^\dagger$}
&70.41 &-
&79.73 &72.93 &68.49 &65.21
\\
\quad + CIB
&\textbf{72.62} (\positive{+2.21}) &-
&\textbf{82.01} (\positive{+2.28}) &\textbf{75.46} (\positive{+2.53}) 
&\textbf{71.05} (\positive{+2.56}) &\textbf{67.71} (\positive{+2.50}) 
\\
UNITER$_{\text{B}}$ \citep{chen2020uniter}\textcolor{red}{$^\dagger$}
&62.68 &70.05
&71.45 &63.72 &59.01 &55.66
\\
\quad + CIB
&64.45 (\positive{+1.77}) &70.91 (\positive{+0.86})
&73.18 (\positive{+1.73}) &66.21 (\positive{+2.49})
&61.88 (\positive{+2.87}) &58.75 (\positive{+3.09})
\\
ALBEF \citep{nips_LiSGJXH21}\textcolor{red}{$^\dagger$}
&65.66 &71.13
&70.89 &65.52 
&61.74 &60.14
\\
\quad + CIB
&68.00 (\positive{+2.34}) &72.43 (\positive{+1.30})
&73.71 (\positive{+2.82}) &67.50 (\positive{+1.98}) 
&63.60 (\positive{+1.86}) &61.72 (\positive{+1.58})
\\ 
\RCRevised{mPLUG$_{\text{B}}$} \citep{li2022mplug}\textcolor{red}{$^\dagger$}
&\RCRevised{65.94} 
&\RCRevised{71.62} 
&\RCRevised{71.01} 
&\RCRevised{67.38} 
&\RCRevised{62.26} 
&\RCRevised{60.46}
\\
\quad \RCRevised{+ CIB}
&\RCRevised{69.02 (\positive{+3.08})} 
&\RCRevised{72.86 (\positive{+1.24})}
&\RCRevised{73.55 (\positive{+2.54})} 
&\RCRevised{70.53 (\positive{+3.15})} 
&\RCRevised{64.73 (\positive{+2.47})} 
&\RCRevised{62.95 (\positive{+2.49})} 
\\ 
\RCRevised{BEiT-3$_{\text{B}}$} \citep{wang2022image}\textcolor{red}{$^\dagger$}
&\RCRevised{67.36} &\RCRevised{73.19}
&\RCRevised{75.96} &\RCRevised{69.73}
&\RCRevised{65.81} &\RCRevised{62.93}
\\
\quad \RCRevised{+ CIB}
&\RCRevised{70.01 (\positive{+2.65})} 
&\RCRevised{\textbf{75.06} (\positive{+1.87})} 
&\RCRevised{78.89 (\positive{+2.93})} 
&\RCRevised{73.01 (\positive{+3.28})} 
&\RCRevised{68.99 (\positive{+3.18})} 
&\RCRevised{65.92 (\positive{+2.99})}
\\ 
\bottomrule
\end{tabularx}
}
\end{center}
\vspace{-4mm}
\end{table*}

\subsubsection{Baseline Pretrained VLMs} 
\RCRevised{
As summarized in Table~\ref{tab:summary_VLP}, we utilize nine pretrained VLMs with three typical transformer architectures as baselines to evaluate the input robustness of our method. 
Specifically, VisualBERT \citep{li2019visualbert}, VL-BERT$_{\text{B}}$ \citep{su2019vl}, and UNITER$_{\text{B}}$ \citep{chen2020uniter} employ single-stream encoders. 
LXMERT \citep{tan2019lxmert}, ViLBERT \citep{lu2019vilbert}, and BEiT-3$_{\text{B}}$ \citep{wang2022image} utilize two-stream encoders. 
VL-T5 \citep{cho2021unifying}, ALBEF \citep{nips_LiSGJXH21}, and mPLUG$_{\text{B}}$ \citep{li2022mplug} incorporate encoder-decoder architectures. 
When applied to the downstream VQA task, mPLUG$_{\text{B}}$, VL-T5, and ALBEF formulate VQA as a text generation task (TG), while the remaining baselines formulate VQA as a multi-answer classification problem (AC). 
These baselines adopt two typical image tokens, namely, the region feature extracted by a pretrained object detector and the patch embedding obtained using a linear projection, and are pretrained on large-scale image-text (IT) data to learn task-agnostic versatile representations. 
}
The pretraining IT datasets include MS COCO caption (COCO) \citep{chen2015microsoft}, Visual Genome (VG) \citep{krishna2017visual}, VQA v2 (VQA) \citep{goyal2017making}, GQA balance version (GQA) \citep{hudson2019gqa}, VG-QA (VGQA) \citep{zhu2016visual7w}, Conceptual Captions (CC) \citep{sharma2018conceptual}, SBU captions (SBU) \citep{ordonez2011im2text}, and Conceptual 12M (CC12M) \citep{changpinyo2021conceptual}. 
Since VQA v2 images originate from the COCO dataset, we follow the work \citep{chen2020uniter} to categorize these pretrained VLMs into in-domain (ID), in-domain and out-of-domain (ID+OOD), and out-of-domain (OOD) groups based on whether they utilize the COCO dataset during the pretraining process.

\subsubsection{Implementation Details} 

\RCRevised{In the subsequent experiments, we maintain the initial configurations of all pretrained VLMs. The region features (visual inputs) of VisualBERT, VL-T5, LXMERT, UNITER$_{\text{B}}$, ViLBERT, and VL-BERT$_{\text{B}}$} are extracted using BUA Faster R-CNN \citep{anderson2018bottom} pretrained on VG \citep{krishna2017visual}. 
The representation dimension $d$ is set to 768. 
\RCRevised{
The configurations of the number of word tokens $L$ (\ie the maximum token length allowed for a question) and image tokens $K$ are detailed in Table~\ref{tab:config}. 
For the only crucial hyperparameter $\beta$ in Eq.~\eqref{eq:CIB_all}, it is set to $1\times10^{-4}$ in all cases except for finetuning VisualBERT and LXMERT, where $\beta$ is set to $5\times10^{-5}$. 
All experiments, except those on ALBEF, BEiT-3$_{\text{B}}$, and mPLUG$_{\text{B}}$ implemented on one NVIDIA A100 40GB GPU, are conducted using PyTorch on one NVIDIA GTX2080Ti 12GB GPU. 
We uniformly utilize an AdamW optimizer with a linear warmup using linear decay and a warmup step of 1000. 
The number of finetuning epochs is 10. 
The configurations of batch size and peak learning rate for each pretrained VLM are shown Table~\ref{tab:config}. 
}
The best model is selected based on the VQA-Score on the mini-split of VQA v2 training set that excludes image-question pairs when evaluating input robustness.

\subsection{Input Robustness Evaluation}

\subsubsection{Robustness against Linguistic Variations}

To evaluate the effectiveness of CIB against linguistic variations, with CIB as the training objective, we finetune pretrained VLMs on VQA v2 training split and report results on VQA-Rephrasings and VQA P2. 
Table~\ref{tab:all_on_vqa_rep} and Table~\ref{tab:all_on_vqa_p2} show the comparisons with existing methods in terms of the VQA-Score as well as the robustness metric of CS($m$).

\begin{table}[!t]
\begin{center}
\caption{Results of robustness against linguistic variations (\ie sentence rephrasing, and word substitution with synonyms and antonyms) on the VQA P2 dataset \citep{whitehead2020learning}
}
\label{tab:all_on_vqa_p2}
\vspace{-1mm}
\setlength{\tabcolsep}{2.0mm}{
\begin{tabularx}{\linewidth}{@{}lll@{}} 
\toprule
Methods
&\text{PER}
&CS(2) 
\\
\midrule[1.pt]
\textit{Data Augmentation}
\\
StackNMN \citep{hu2018explainable} 
&63.30 
&{66.20} 
\\
\quad + Q3R \citep{whitehead2020learning}  
&66.90 (\positive{+3.30}) 
&{72.20} (\positive{+6.00}) 
\\
HybridNet \citep{whitehead2020learning} 
&63.30 
&{66.60} 
\\
\quad + Q3R \citep{whitehead2020learning} 
&67.00 (\positive{+4.00})  
&{72.50} (\positive{+5.90})  
\\
XNM \citep{shi2019xnm} 
&64.70 
&{68.80} 
\\
\quad + Q3R \citep{whitehead2020learning} 
&68.10 (\positive{+3.40})  
&{74.40} (\positive{+5.60})  
\\
\midrule
\textit{w/o Data Augmentation}
\\
VisualBERT \citep{li2019visualbert}\textcolor{red}{$^\dagger$} 
&68.23 &72.34
\\
\quad + CIB 
&69.92 (\positive{+1.69}) 
&73.83 (\positive{+1.49})
\\ 
ViLBERT \citep{lu2019vilbert}\textcolor{red}{$^\dagger$} 
&67.18 &71.39
\\
\quad + CIB
&69.92 (\positive{+2.74}) 
&73.98 (\positive{+2.59}) 
\\ 
VL-BERT$_{\text{B}}$ \citep{su2019vl}\textcolor{red}{$^\dagger$} 
&68.36 &72.52
\\
\quad + CIB
&69.82 (\positive{+1.46}) 
&73.88 (\positive{+1.36}) 
\\ 
VL-T5 \citep{cho2021unifying}\textcolor{red}{$^\dagger$}
&71.63 &77.34
\\
\quad + CIB
&73.47 (\positive{+1.84}) 
&78.99 (\positive{+1.65}) 
\\ 
LXMERT \citep{tan2019lxmert}\textcolor{red}{$^\dagger$}
&77.30 &82.96
\\
\quad + CIB  
&\textbf{78.93} (\positive{+1.63}) 
&\textbf{85.07} (\positive{+2.11})
\\ 
UNITER$_\text{B}$ \citep{chen2020uniter}\textcolor{red}{$^\dagger$}
&70.36 &74.36 
\\
\quad + CIB  
&71.30 (\positive{+0.94}) 
&75.91 (\positive{+1.55}) 
\\ 
ALBEF \citep{nips_LiSGJXH21}\textcolor{red}{$^\dagger$}
&71.36 &76.00
\\
\quad + CIB
&72.84 (\positive{+1.48}) 
&77.46 (\positive{+1.46}) 
\\ 
\RCRevised{mPLUG$_{\text{B}}$} \citep{li2022mplug}\textcolor{red}{$^\dagger$}
&\RCRevised{71.95} 
&\RCRevised{76.75}
\\
\quad \RCRevised{+ CIB}
&\RCRevised{73.11 (\positive{+1.16})} 
&\RCRevised{78.09 (\positive{+1.34})} 
\\ 
\RCRevised{BEiT-3$_{\text{B}}$} \citep{wang2022image}\textcolor{red}{$^\dagger$}
&\RCRevised{73.65}
&\RCRevised{78.56}
\\
\quad \RCRevised{+ CIB}
&\RCRevised{75.22 (\positive{+1.57})} 
&\RCRevised{81.28 (\positive{+2.72})} 
\\ 
\bottomrule
\end{tabularx}
}
\end{center}
\vspace{-4mm}
\end{table}

\begin{table*}[!t]
\begin{center}
\caption{Results of robustness against visual variations on IV-VQA and CV-VQA \citep{agarwal2020towards}
}
\label{tab:all_on_iv_vqa}
\vspace{-1mm}
\setlength{\tabcolsep}{2.7mm}{
\begin{tabularx}{\linewidth}{@{}lllllll@{}}
\toprule
\multirow{2}{*}{Methods}
&\multicolumn{3}{c}{IV-VQA} 
&\multicolumn{3}{c}{CV-VQA} 
\\
\cmidrule(lr){2-4} \cmidrule(l){5-7}
&PER~$\uparrow$
&ORI~$\uparrow$
&\#flips~$\downarrow$
&PER~$\uparrow$
&ORI~$\uparrow$
&\#flips~$\downarrow$
\\ 
\midrule[1pt]
CL \citep{Lu2015} 
&- &60.21 &17.89 
&- &39.38 &81.41 
\\ 
SNMN \citep{hu2018explainable} 
&- &66.04 &~~6.52 
&- &47.95 &78.92 
\\ 
SAAA \citep{kazemi2017show}%
&- &70.26 &~~7.85 
&- &49.90 &78.44 
\\ 
UNITER$_\text{B}$ \citep{chen2020uniter}
&- &- &~~8.47
&- &- &40.67 
\\
\quad MANGO$_\text{B}$ \citep{li2020closer} 
&- &- &~~7.32 (\positive{+1.15}) 
&- &- &38.11 (\positive{+2.56}) 
\\ 
VILLA$_{\text{B}}$ \citep{gan2020large}
&- &- &~~7.07
&- &- &38.28
\\
\quad MANGO$_{\text{VB}}$ \citep{li2020closer} 
&- &- &~~7.43 (\positive{+0.36})
&- &- &38.25 (\textcolor{red}{-0.03})
\\
\midrule 
VisualBERT \citep{li2019visualbert}\textcolor{red}{$^\dagger$}
&46.04 &81.99 &26.84  
&30.48 &76.30 &30.13
\\
\quad + CIB 
&47.81 (\positive{+1.77}) 
&83.48 (\positive{+1.49}) 
&23.91 (\positive{+2.93})
&32.46 (\positive{+1.98}) 
&77.09 (\positive{+0.79}) 
&27.98 (\positive{+2.15})
\\ 
ViLBERT \citep{lu2019vilbert}\textcolor{red}{$^\dagger$} 
&72.37 &81.73 &11.98
&32.24 &70.70 &35.43
\\
\quad + CIB
&74.67 (\positive{+2.30}) 
&83.35 (\positive{+1.62}) 
&10.85 (\positive{+1.13})
&35.33 (\positive{+3.09}) 
&71.11 (\positive{+0.41}) 
&34.01 (\positive{+1.42})
\\ 
VL-BERT$_\text{B}$ \citep{su2019vl}\textcolor{red}{$^\dagger$}
&72.42 &82.35 &12.58
&33.52 &71.00 &34.28
\\
\quad + CIB
&73.66 (\positive{+1.24}) 
&83.37 (\positive{+1.02}) 
&11.00 (\positive{+1.58})
&35.29 (\positive{+1.77}) 
&72.70 (\positive{+1.70}) 
&32.09 (\positive{+2.19})
\\
VL-T5 \citep{cho2021unifying}\textcolor{red}{$^\dagger$}
&75.73 
&--
&~~9.76
&35.09 
&--
&33.43
\\
\quad + CIB
&76.46 (\positive{+0.73}) 
&--
&~~8.55 (\positive{+1.21})
&42.55 (\positive{+7.46}) 
&--
&30.42 (\positive{+3.01}) 
\\
LXMERT \citep{tan2019lxmert}\textcolor{red}{$^\dagger$} 
&77.83 
&--
&12.67
&38.86 
&--
&32.80
\\
\quad + CIB 
&78.57 (\positive{+0.74}) 
&--
&11.64 (\positive{+1.03})
&40.47 (\positive{+1.61}) 
&--
&30.37 (\positive{+2.43})
\\ 
UNITER$_\text{B}$ \citep{chen2020uniter}\textcolor{red}{$^\dagger$}
&75.71 &84.56 &11.77
&42.60 &78.27 &29.67
\\
\quad + CIB 
&76.63 (\positive{+0.92}) 
&86.05 (\positive{+1.49}) 
&~~9.80 (\positive{+1.97})
&46.92 (\positive{+4.32}) 
&79.89 (\positive{+1.62}) 
&27.99 (\positive{+1.68}) 
\\
ALBEF \citep{nips_LiSGJXH21}\textcolor{red}{$^\dagger$}
&85.63 &87.87 &12.00
&50.00 &80.00 &28.71
\\
\quad + CIB
&{87.21} (\positive{+1.58}) 
&{88.88} (\positive{+1.01}) 
&{~~9.16} (\positive{+2.84})
&{51.78} (\positive{+1.78}) 
&{81.45} (\positive{+1.45}) 
&{25.76} (\positive{+2.95})
\\
\RCRevised{mPLUG$_{\text{B}}$} \citep{li2022mplug}\textcolor{red}{$^\dagger$}
&\RCRevised{86.96}
&\RCRevised{89.40}
&\RCRevised{13.16}
&\RCRevised{53.92}
&\RCRevised{78.49}
&\RCRevised{25.50}
\\
\quad \RCRevised{+ CIB}
&\RCRevised{88.47 (\positive{+1.51})} 
&\RCRevised{90.33 (\positive{+0.93})} 
&\RCRevised{10.59 (\positive{+2.57})} 
&\RCRevised{55.20 (\positive{+1.28})} 
&\RCRevised{79.43 (\positive{+0.94})} 
&\RCRevised{24.07 (\positive{+1.43})} 
\\
\RCRevised{BEiT-3$_{\text{B}}$} \citep{wang2022image}\textcolor{red}{$^\dagger$}
&\RCRevised{87.96}
&\RCRevised{89.38}
&\RCRevised{~~9.08}
&\RCRevised{55.17}
&\RCRevised{79.99}
&\RCRevised{24.08}
\\
\quad \RCRevised{+ CIB}
&\RCRevised{\textbf{90.00} (\positive{+2.04})} 
&\RCRevised{90.64 (\positive{+1.26})} 
&\RCRevised{~~\textbf{5.40} (\positive{+3.68})} 
&\RCRevised{\textbf{57.95} (\positive{+2.78})} 
&\RCRevised{81.00 (\positive{+1.01})} 
&\RCRevised{\textbf{23.23} (\positive{+0.85})} 
\\ 
\bottomrule
\end{tabularx}
}
\end{center}
\vspace{-4mm}
\end{table*}

\pgraph{Result on VQA-Rephrasings}
We first compare the proposed CIB with existing methods: CC \citep{shah2019cycle}, ConClaT \citep{kant2020contrast}, and MANGO \citep{li2020closer}. 
Specifically, both CC and ConClaT augment training datasets online by training a question generation model to generate paraphrases of questions. 
To effectively leverage augmented data and enhance model robustness to linguistic variations, CC considers cycle consistency between the question and its rephrasings, while ConClaT jointly optimizes contrastive and cross-entropy losses. 
CC considers three baseline VQA models, \ie BUTD \citep{anderson2018bottom}, Pythia \citep{jiang2018pythia}, and BAN \citep{kim2018bilinear}. 
ConClaT uses MMT \citep{kant2020contrast}, a modified version of UNITER, as its baseline. 
MANGO employs UNITER \citep{chen2020uniter} and VILLA \citep{gan2020large} as baseline models and adopts adversarial training to enhance model robustness. 
As shown in Table~\ref{tab:all_on_vqa_rep}, \footnote{VL-T5 and LXMERT utilize some examples from the VQA v2 validation set in the pretraining VQA task, resulting in an unreliable VQA-Score for ORI, thus we do not report the ORI performance.} the results on nine pretrained VLMs consistently show that compared to baselines (\ie finetuning pretrained VLMs with only the task-related loss for answer prediction\textcolor{red}{$^\dagger$}), using CIB as the training objective for VQA models can significantly improve their robustness to linguistic variations. 
This finding suggests that it is feasible to encourage models to learn more compact and robust representations from an information-theoretic perspective. 
In comparison with state-of-the-art methods, adapting LXMERT with CIB achieves the best performance across all metrics. 
\RCRevised{This performance advantage can be attributed to LXMERT considering the VQA training objective during pretraining, which reduces the gap between upstream and downstream objectives. }
In addition, we observe that the data augmentation based method (CC) yields greater improvements in the metric of CS(4). 
However, without data augmentation, the average improvement of our method is more substantial.

\pgraph{Result on VQA P2} 
We next compare our method with the existing method Q3R \citep{whitehead2020learning} on VQA P2. 
Q3R augments training data by creating linguistic variations such as synonymous, paraphrastic, and antonymous of input questions, and regularizes the visual reasoning process between the question and its generated questions. 
Q3R utilizes three baseline models: StackNMN \citep{hu2018explainable}, HybridNet \citep{whitehead2020learning}, and XNM \citep{shi2019xnm}. 
The results in Table~\ref{tab:all_on_vqa_p2} indicate that finetuning pretrained VLMs with the proposed CIB can markedly improve their robustness against question variations on VQA P2. 
Moreover, finetuning LXMERT with CIB also achieves the best performance on VQA P2. 
In addition, the data augmentation-based method (Q3R) continues to exhibit superiority in improving the input robustness of baseline VQA models.

\subsubsection{Robustness against Visual Variations} 

We evaluate the robustness of our method against visual variations on IV-VQA and CV-VQA. 
Table~\ref{tab:all_on_iv_vqa} shows the comparisons with existing methods in the metrics of VQA-Score and \#flips. 
CL (a simple CNN+LSTM model) \citep{Lu2015}, SNMN (an attention-based method) \citep{hu2018explainable}, and SAAA (a compositional model) \citep{kazemi2017show} are benchmarked by Agarwal~\etal \citep{agarwal2020towards}. 
MANGO exploits adversarial training to improve the robustness of pretrained VLMs (UNITER \citep{chen2020uniter} and VILLA \citep{gan2020large}) against visual variations. 
The results in Table~\ref{tab:all_on_iv_vqa} show that significant improvements are achieved across all metrics and baselines on both IV-VQA and CV-VQA, suggesting the effectiveness of CIB in improving robustness against visual variations. 
\RCRevised{
Moreover, we observe that pretrained VLMs using raw images as visual inputs (\eg BEiT-3$_{\text{B}}$, mPLUG$_{\text{B}}$, and ALBEF) exhibit superior performance in defending against visual variations compared to those pretrained VLMs (\eg VisualBERT, LXMERT, and UNITER$_{\text{B}}$) that employ object-level region features as visual inputs. 
This can be attributed to the fact that pre-extracted region features lose some image information, which hinders VQA models in comprehending and retrieving visual content according to a given question. 
}

\subsubsection{Robustness against Multimodal Shortcut Learning} 

To demonstrate the ability of CIB to defend against multimodal shortcuts present in input images and questions, we conduct experiments on VQA-CE and compare our methods with existing approaches. 
Results are summarized in Table~\ref{tab:all_on_vqa_ce}. 
The compared methods in the table can be broadly classified into two groups: (\emph{i}) plain models (SAN \citep{yang2016stacked}, BLOCK \citep{ben2019block}, VilBERT \citep{lu2019vilbert}, and BUTD \citep{anderson2018bottom}), and (\emph{ii}) bias-reduction methods (RUBi \citep{cadene2019rubi}, LMH + RMFE \citep{gat2020removing}, ESR \citep{shrestha2020negative}, LMH \citep{clark2019don}, LfF \citep{nam2020learning}, LMH + CSS \citep{chen2020counterfactual}, and RandImg \citep{teney2020value}). 
These experimental results are cited from the work \citep{dancette2021beyond}. 
As shown in Table~\ref{tab:all_on_vqa_ce}, finetuning baseline pretrained VLMs with CIB achieves significant improvements and outperforms bias-reduction methods by a considerable margin, particularly on counterexamples. 
These results suggest that the proposed CIB is more effective at alleviating the spurious correlations between representations and reducing shortcut learning involved in multimodal inputs.

\begin{table}[!t]
\begin{center}
\caption{Results of robustness against multimodal shortcut learning on the VQA-CE dataset \citep{dancette2021beyond}}
\label{tab:all_on_vqa_ce}
\vspace{-1mm}
\setlength{\tabcolsep}{1.47mm}{
\begin{tabularx}{\linewidth}{@{}lll@{}}
\toprule
\multirow{2}{*}{Methods}
&\multicolumn{2}{c}{VQA-Score}
\\ 
\cmidrule(l){2-3}
&CE &Easy
\\
\midrule[1pt]
Shortcuts \citep{dancette2021beyond} &0.00 &61.13 
\\
SAN \citep{yang2016stacked} &26.64 &68.45
\\
BLOCK \citep{ben2019block} &32.91 &77.65 
\\
VilBERT \citep{lu2019vilbert} &39.24 &80.50 
\\
BUTD \citep{anderson2018bottom} &33.91 &76.69 
\\
\quad + RUBi \citep{cadene2019rubi} 
&32.25 (\textcolor{red}{-1.66}) &75.03 (\textcolor{red}{-1.66}) 
\\
\quad + LMH + RMFE \citep{gat2020removing} 
&33.14 (\textcolor{red}{-0.77}) &73.32 (\textcolor{red}{-3.37}) 
\\
\quad + ESR \citep{shrestha2020negative} 
&33.26 (\textcolor{red}{-0.65}) &76.18 (\textcolor{red}{-0.51}) 
\\
\quad + LMH \citep{clark2019don} 
&34.26 (\positive{+0.35}) &73.12 (\textcolor{red}{-3.57}) 
\\
\quad + LfF \citep{nam2020learning} 
&34.27 (\positive{+0.36}) &76.60 (\textcolor{red}{-0.09})  
\\
\quad + LMH + CSS \citep{chen2020counterfactual} 
&34.36 (\positive{+0.45}) &62.08 (\textcolor{red}{-14.61}) 
\\
\quad + RandImg \citep{teney2020value} 
&34.41 (\positive{+0.50}) &76.21 (\textcolor{red}{-0.48}) 
\\ 
\midrule
ViLBERT \citep{lu2019vilbert}\textcolor{red}{$^\dagger$} 
&38.91 &80.96
\\
\quad + CIB
&41.24 (\positive{+2.33}) 
&82.96 (\positive{+2.00})
\\ 
VL-BERT$_{\text{B}}$ \citep{su2019vl}\textcolor{red}{$^\dagger$} 
&36.56 &80.66
\\
\quad + CIB
&38.24 (\positive{+1.68}) 
&82.00 (\positive{+1.34})
\\ 
VisualBERT \citep{li2019visualbert}\textcolor{red}{$^\dagger$}
&38.75 &79.42
\\
\quad + CIB
&40.86 (\positive{+2.11}) 
&81.25 (\positive{+1.83}) 
\\ 
VL-T5 \citep{cho2021unifying}\textcolor{red}{$^\dagger$}
&45.41 &86.05
\\
\quad + CIB
&47.60 (\positive{+2.19}) 
&88.00 (\positive{+1.95})
\\ 
LXMERT \citep{tan2019lxmert}\textcolor{red}{$^\dagger$} 
&53.61 &87.63
\\
\quad + CIB 
&\textbf{57.14} (\positive{+3.53}) 
&\textbf{89.21} (\positive{+1.68})
\\ 
UNITER$_\text{B}$ \citep{chen2020uniter}\textcolor{red}{$^\dagger$}
&40.64 &81.75 
\\
\quad + CIB 
&42.03 (\positive{+1.39}) 
&82.48 (\positive{+0.73})
\\ 
ALBEF \citep{nips_LiSGJXH21}\textcolor{red}{$^\dagger$}
&45.39 &83.88
\\
\quad + CIB
&47.87 (\positive{+2.48}) 
&86.00 (\positive{+2.12})
\\ 
\RCRevised{mPLUG$_{\text{B}}$} \citep{li2022mplug}\textcolor{red}{$^\dagger$}
&\RCRevised{45.74} 
&\RCRevised{84.07}
\\
\quad \RCRevised{+ CIB}
&\RCRevised{47.73 (\positive{+1.99})} 
&\RCRevised{85.25 (\positive{+1.18})} 
\\ 
\RCRevised{BEiT-3$_{\text{B}}$} \citep{wang2022image}\textcolor{red}{$^\dagger$}
&\RCRevised{47.15} 
&\RCRevised{84.44}
\\
\quad \RCRevised{+ CIB}
&\RCRevised{50.38 (\positive{+3.23})} 
&\RCRevised{85.92 (\positive{+1.48})} 
\\ 
\bottomrule
\end{tabularx}
}
\end{center}
\vspace{-4mm}
\end{table}

\subsection{Ablation Studies} 
\label{sec:abl}


\subsubsection{Comparison with Alternative CIB Bounds} 
\label{sec:abl_cib}

When finetuning pretrained VLMs with CIB, $I(Y;T)$ is regarded as the task-related objective, while $I(X^v, X^l; T^v, T^l)$ serves as a MI regularizer to constrain representation compactness and pursue more robust representations. 
As stated in Theorem~\ref{thm:JMI}, the upper bound of $I(X^v, X^l; T^v, T^l)$ consists of four terms: $I(X^v; T^v)$, $I(X^l; T^l)$, $- I(T^v; T^l)$, and $\SKL$. 
To analyze the contribution of different terms to CIB, we perform an ablation study on different meaningful combinations of these terms, that is, provable upper bounds of $I(X^v, X^l; T^v, T^l)$, on VQA-Rephrasings using LXMERT, UNITER$_{\text{B}}$, and ALBEF as baseline pretrained VLMs. 
Specifically, the regularizer upper bound has three other meaningful alternatives: 
(\emph{i}) $\frac{3}{2}[I(X^v; T^v) + I(X^l; T^l)]$, 
(\emph{ii}) $-I(T^v; T^l) + \SKL$, and 
(\emph{iii}) $I(X^v; T^v) + I(X^l; T^l) + \SKL$ 
(\cf Section~\ref{sec:alternative_ub} for proofs).  
Table~\ref{tab:comp_ibs} presents results on VQA-Rephrasings. 
Overall, the ablation results on different bounds are consistent, indicating that CIB with any meaningful upper bounds can markedly improve the performance of baseline pretrained VLMs. 
However, CIB with our derived upper bound performs best, empirically demonstrating that the bound in Theorem~\ref{thm:JMI} is a tighter and more precise bound. 
\RCRevised{Furthermore, the comparison between upper bound (\emph{iii}) $I(X^v; T^v) + I(X^l; T^l) + \SKL$ and our upper bound $I(X^v; T^v) + I(X^l; T^l) - I(T^v; T^l) + \SKL$ suggests that CIB can effectively facilitate the correlation between visual and linguistic representations and modality alignment by maximizing $I(T^v; T^l)$. }

\subsubsection{Impact of MI Estimator on CIB} 

In practice, any sample-based upper bound estimator of MI can be utilized to approximate $I(X^v; T^v)$ and $I(X^l; T^l)$, and any differentiable MI lower bound estimator can be applied to approach $I(T^v; T^l)$. 
To analyze the impact of different MI estimators on CIB, we consider the following experimental settings: 
(\emph{i}) We alternately utilize L1Out \citep{poole2019variational} instead of CLUB \citep{cheng2020club} as the estimator of MI upper bound to approximate $I(X^v; T^v)$ and $I(X^l; T^l)$. 
(\emph{ii}) We approximate $I(T^v; T^l)$ with the three other estimators of MI lower bound, \ie InfoNCE \citep{oord2018representation}, NWJ \citep{nguyen2010estimating}, and MINE \citep{belghazi2018mine}. 
Table~\ref{tab:mi_estimator} presents comparisons between different MI estimators on VQA-Rephrasings using LXMERT, UNITER$_{\text{B}}$, and ALBEF as baselines. 
These results consistently demonstrate that CIB can effectively improve the performance of baselines with different transformer architectures and that the effectiveness of CIB does not depend on a specific MI estimator.

\begin{table}[!t]
\begin{center}
\caption{Comparison between different CIB bounds}
\label{tab:comp_ibs}
\vspace{-1mm}
\setlength{\tabcolsep}{1.2mm}{
\begin{tabularx}{\linewidth}{@{}lcccccc@{}}
\toprule
\multirow{2}{*}{VLMs} 
&\multicolumn{5}{c}{CIB Terms}
&\multirow{2}{*}{PER}
\\ 
\cmidrule(lr){2-6}
&\RCRevised{$I(Y; T)$}
&\RCRevised{$I(X^v; T^v)$}
&\RCRevised{$I(X^l; T^l)$}
&\RCRevised{$-I(T^v; T^l)$}
&\RCRevised{$\SKL$}
&
\\ 
\midrule[1pt]
\multirow{5}{*}{\makecell[l]{LXMERT}}
&\RCRevised{\checkmark} &\multicolumn{4}{c}{} 
&70.41
\\ 
&\RCRevised{\checkmark} &\checkmark &\checkmark & & 
&72.17
\\ 
&\RCRevised{\checkmark} & & &\checkmark &\checkmark 
&72.07
\\ 
&\RCRevised{\checkmark} &\checkmark &\checkmark & &\checkmark 
&72.28 
\\ 
&\RCRevised{\checkmark} &\checkmark &\checkmark &\checkmark &\checkmark 
&\textbf{72.62} 
\\ 
\midrule
\multirow{5}{*}{\makecell[l]{UNITER$_{\text{B}}$}}
&\RCRevised{\checkmark} &\multicolumn{4}{c}{} 
&62.68
\\ 
&\RCRevised{\checkmark} &\checkmark &\checkmark & & 
&64.07
\\ 
&\RCRevised{\checkmark} & & &\checkmark &\checkmark 
&64.11
\\ 
&\RCRevised{\checkmark} &\checkmark &\checkmark & &\checkmark 
&63.23
\\ 
&\RCRevised{\checkmark} &\checkmark &\checkmark &\checkmark &\checkmark 
&\textbf{{64.45}}
\\ 
\midrule
\multirow{5}{*}{\makecell[l]{ALBEF}}
&\RCRevised{\checkmark} &\multicolumn{4}{c}{} 
&65.66
\\ 
&\RCRevised{\checkmark} &\checkmark &\checkmark & & 
&67.11
\\ 
&\RCRevised{\checkmark} & & &\checkmark &\checkmark 
&67.00
\\ 
&\RCRevised{\checkmark} &\checkmark &\checkmark & &\checkmark 
&66.84 
\\ 
&\RCRevised{\checkmark} &\checkmark &\checkmark &\checkmark &\checkmark 
&\textbf{68.00} 
\\ 
\bottomrule
\end{tabularx}
}
\end{center}
\vspace{-4mm}
\end{table}

\begin{table}[!t]
\begin{center}
\caption{Impact of different MI estimators on CIB}
\label{tab:mi_estimator}
\vspace{-1mm}
\setlength{\tabcolsep}{5.2mm}{
\begin{tabularx}{\linewidth}{@{}lccc@{}}
\toprule 
\multirow{2}{*}{VLMs} 
&\multicolumn{2}{c}{MI Estimator}
&\multirow{2}{*}{PER} 
\\ 
\cmidrule(lr){2-3}
&Upper Bound
&Lowe Bound
&
\\ 
\midrule[1pt]
\multirow{5}{*}{\makecell[l]{LXMERT}}
&\multicolumn{2}{c}{} 
&70.41
\\
&L1Out
&NWJ
&72.31  
\\ 
&CLUB
&InfoNCE
&72.34
\\ 
&CLUB
&MINE
&72.48 
\\ 
&CLUB
&NWJ
&\textbf{72.62} 
\\ 
\midrule
\multirow{5}{*}{\makecell[l]{UNITER$_{\text{B}}$}}
&\multicolumn{2}{c}{} 
&62.68
\\
&L1Out
&NWJ
&64.27  
\\ 
&CLUB
&InfoNCE
&64.14
\\ 
&CLUB
&MINE
&64.32 
\\ 
&CLUB
&NWJ
&\textbf{64.45} 
\\ 
\midrule
\multirow{5}{*}{\makecell[l]{ALBEF}}
&\multicolumn{2}{c}{} 
&65.66
\\
&L1Out
&NWJ 
&67.68  
\\ 
&CLUB
&InfoNCE
&\textbf{68.00}
\\ 
&CLUB
&MINE
&67.91 
\\ 
&CLUB 
&NWJ
&\textbf{68.00} 
\\ 
\bottomrule
\end{tabularx}
}
\end{center}
\vspace{-4mm}
\end{table}

\begin{figure*}[!t]
\centering 
\subfigure[LXMERT]{
\label{fig:a}
\includegraphics[width=0.31\textwidth]{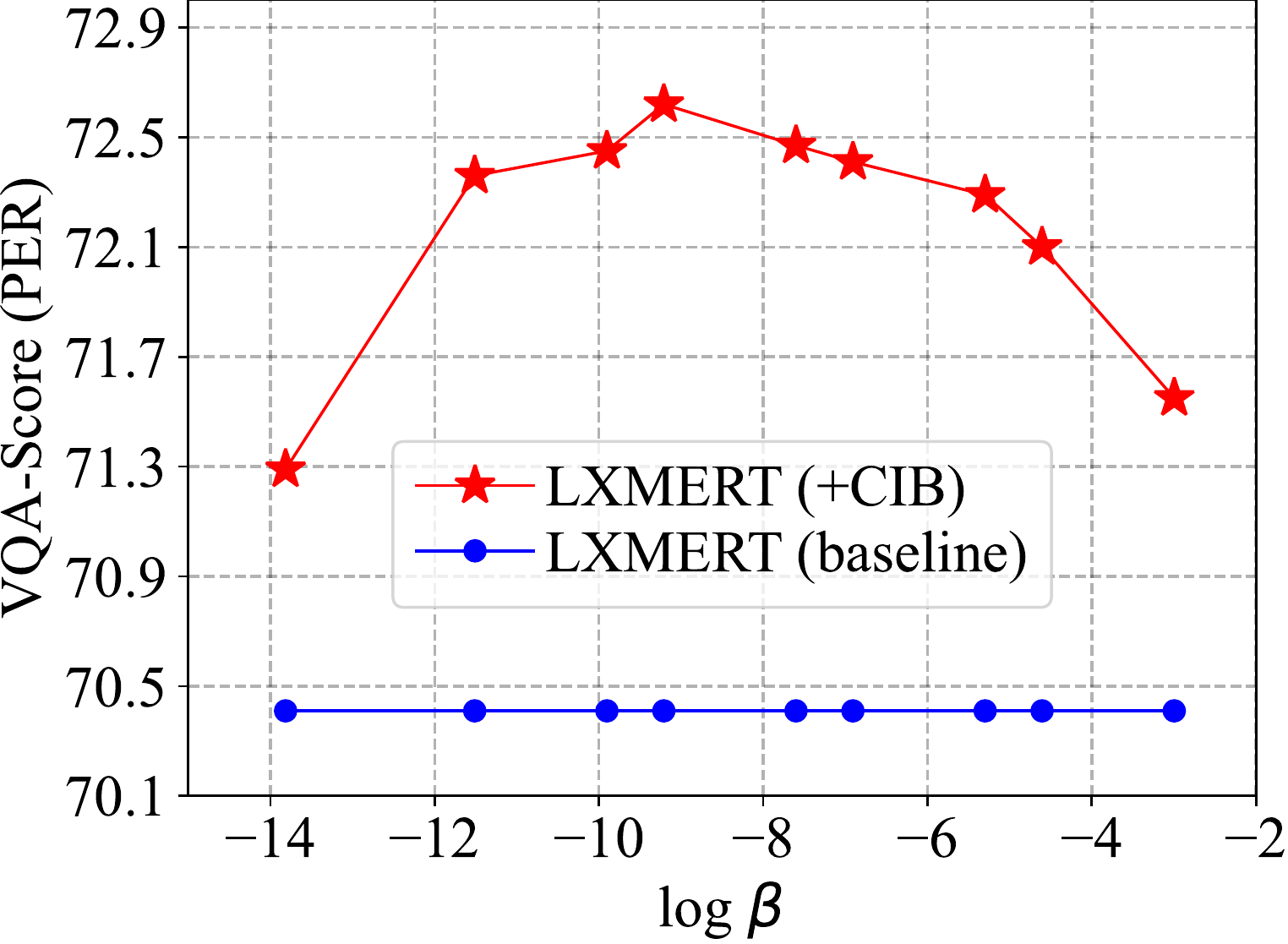}
}
\quad
\subfigure[UNITER$_{\text{B}}$]{\label{fig:b}\includegraphics[width=0.31\textwidth]{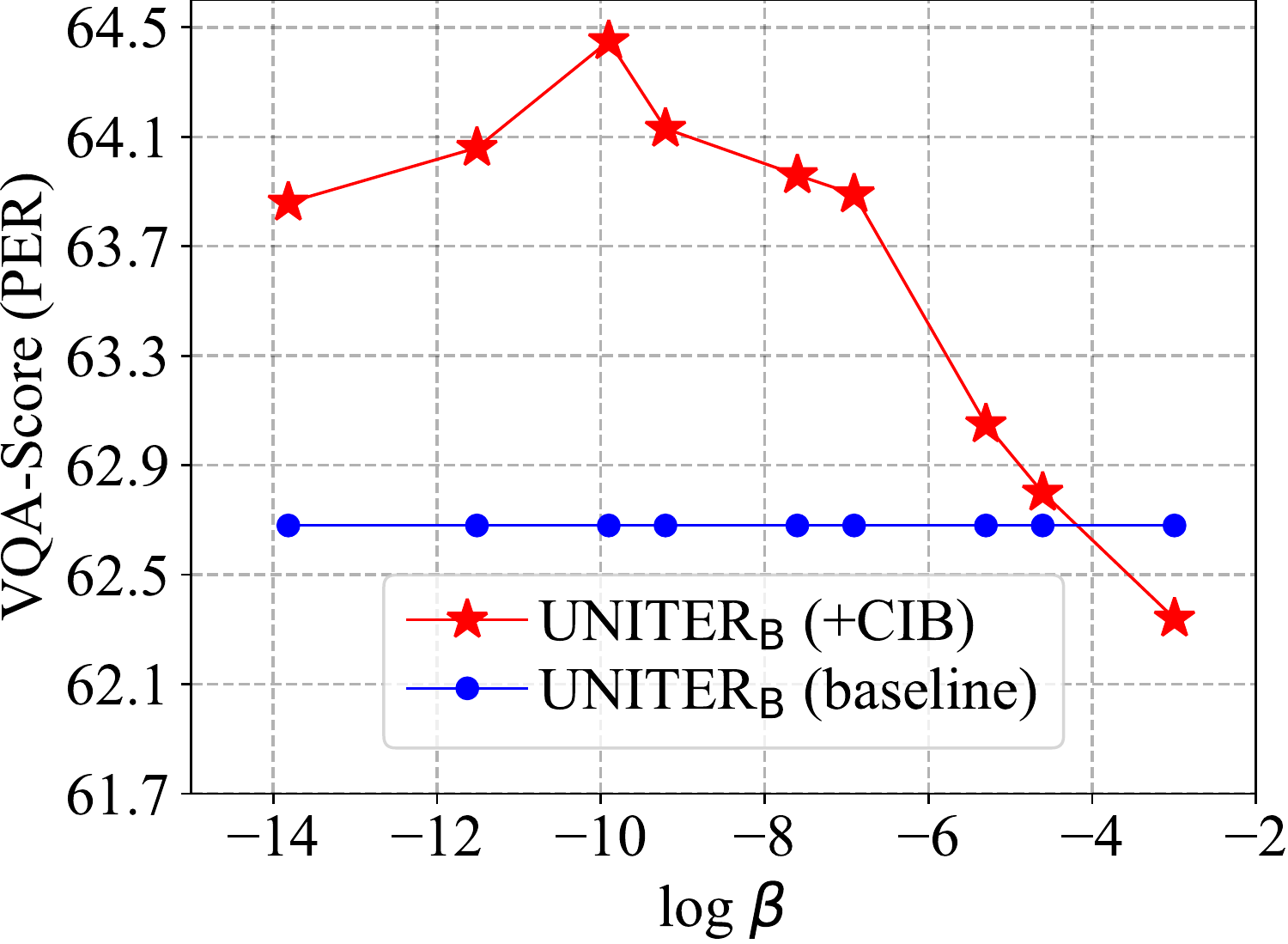}
}
\quad
\subfigure[ALBEF]{\label{fig:c}\includegraphics[width=0.31\textwidth]{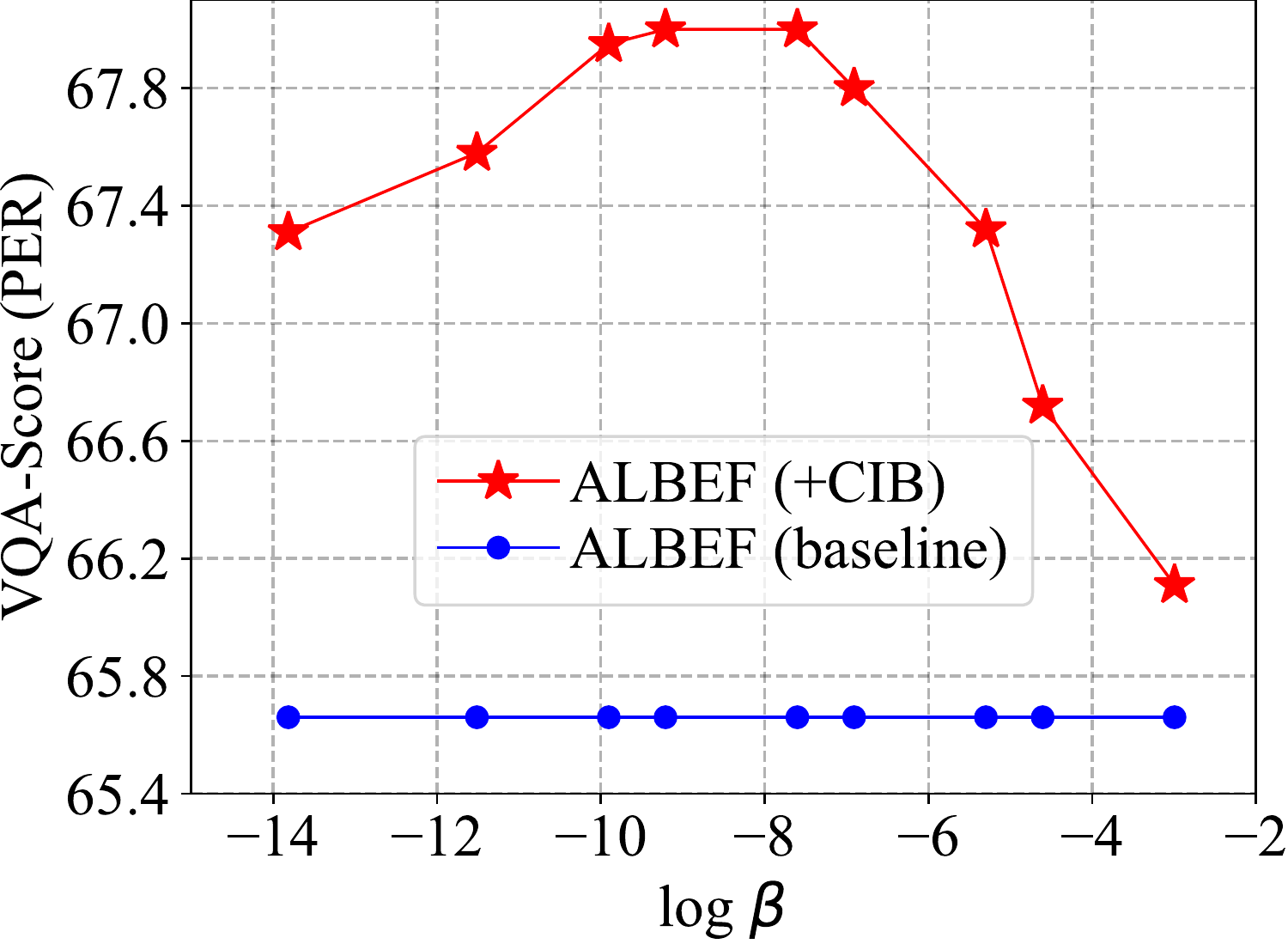}
}
\vspace{-2mm}
\caption{Variation curve of VQA-Score (PER) on VQA-Rephrasings as $\log\beta$ increases}
\label{fig:hp_beta}
\vspace{-4mm}
\end{figure*}

\subsubsection{Impact of Hyperparameter on CIB}

When using CIB as the training objective to adapt pretrained VLMs to the downstream VQA task, $\beta$ controls the tradeoff between redundancy and compression in representations, which is the crucial hyperparameter. 
Consequently, we perform a grid search for $\beta$. 
Specifically, we consider the following values: $\beta \in [1\times10^{-6}, 1\times10^{-5}, 5\times10^{-5}, 1\times10^{-4}, 5\times10^{-4}, 1\times10^{-3}, 5\times10^{-3}, 1\times10^{-2}, 5\times10^{-2}]$. 
\figref{hp_beta} illustrates the variation curve of VQA-Score (PER) on VQA-Rephrasings with increasing $\log\beta$. 
We observe that the performance starts to boost when $\beta$ is quite small, indicating the effectiveness of CIB in improving the performance of baseline pretrained VLMs. 
When $\beta$ increases to $5\times10^{-5}$, $1\times10^{-4}$, and $1\times10^{-4}$, UNITER$_\text{B}$, LXMERT, and ALBEF respectively achieve the best performance. 
Beyond that point, the performance typically begins to degrade, suggesting that extremely compressed representations of pretrained VLMs may start to compromise model performance.

\begin{table}[!t]
\begin{center}
\caption{Impact of different internal representations obtained by the two-stream pretrained VLMs on CIB}
\label{tab:lxmert_var}
\vspace{-1mm}
\setlength{\tabcolsep}{7.2mm}{
\begin{tabularx}{\linewidth}{@{}llc@{}}
\toprule 
VLMs
&$\mathcal{L}_{\text{CIB}}$
&PER
\\ 
\midrule[1pt]
\multirow{3}{*}{\makecell[l]{LXMERT}}
&\multicolumn{1}{l}{\RCRevised{$I(Y; T)$}} 
&70.41 
\\
&$I(Y; T) - \beta I(X^v, X^l; T^{'v}, T^{'l})$
&72.53
\\
&$I(Y; T) - \beta I(X^v, X^l; T^v, T^l)$ 
&\textbf{72.62} 
\\ 
\midrule
\multirow{3}{*}{\makecell[l]{ViLBERT}}
&\multicolumn{1}{l}{\RCRevised{$I(Y; T)$}} 
&59.16 
\\
&$I(Y; T) - \beta I(X^v, X^l; T^{'v}, T^{'l})$
&62.23
\\
&$I(Y; T) - \beta I(X^v, X^l; T^v, T^l)$ 
&\textbf{62.28} 
\\ 
\bottomrule
\end{tabularx}
}
\end{center}
\vspace{-4mm}
\end{table}

\begin{table}[!t]
\begin{center}
\caption{
Results on VQA v2 test-dev \citep{goyal2017making} under standard and clean dataset setups}
\label{tab:comp_vqa_v2_all} 
\vspace{-1mm}
\setlength{\tabcolsep}{1.85mm}{
\begin{tabularx}{\linewidth}{@{}lrr}
\toprule
\multirow{2}{*}{VLMs} 
&\multicolumn{2}{c}{VQA-Score}
\\
\cmidrule(l){2-3}
&Baseline &+ CIB  
\\ 
\midrule[1pt] 
VisualBERT \citep{li2019visualbert} 
&70.80 (70.46\textcolor{red}{$^{\dagger}$}) 
&\cellcolor{gray!15}71.62 (\positive{+1.16}) 
\\ 
VL-T5 \citep{cho2021unifying}
&- (70.23\textcolor{red}{$^{\dagger}$}) 
&71.14\cellcolor{gray!15} (\positive{+0.91}) 
\\
LXMERT \citep{tan2019lxmert} 
&72.42 (72.58\textcolor{red}{$^{\dagger}$})
&\cellcolor{gray!15}72.99 (\positive{+0.41})
\\ 
UNITER$_{\text{B}}$ \citep{chen2020uniter} 
&72.70 (71.63\textcolor{red}{$^{\dagger}$})
&\cellcolor{gray!15}72.11 (\positive{+0.48}) 
\\ 
ALBEF \citep{nips_LiSGJXH21}
&74.54 (74.54\textcolor{red}{$^{\dagger}$}) 
&76.27\cellcolor{gray!15} (\positive{+1.73}) 
\\
ViLBERT \citep{lu2019vilbert} 
&70.55 (70.55\textcolor{red}{$^{\dagger}$}) 
&\cellcolor{gray!15}71.00 (\positive{+0.45})
\\ 
VL-BERT$_{\text{B}}$ \citep{su2019vl} 
&71.16 (71.20\textcolor{red}{$^{\dagger}$}) 
&\cellcolor{gray!15}71.59 (\positive{+0.39}) 
\\ 
\bottomrule 
\end{tabularx}
}
\end{center} 
\vspace{-4mm}
\end{table}

\subsubsection{Impact of Internal Representation on CIB} 

As illustrated in \figref{ts}, for pretrained VLMs with two-stream encoders (\eg LXMERT and ViLBERT), there is an alternative option for internal representations, \ie $T=[T^{'v}, T^{'l}]$, which are the visual and linguistic representations after the vision transformer layers ($f_{\theta^{\text{VTran}}}$) and language transformer layers ($f_{\theta^{\text{LTran}}}$). 
When finetuning the two-stream pretrained VLMs with CIB, we analyze the impact of different internal representations by replacing the original $T=[T^{v}, T^{l}]$ in $\mathcal{L}_{\text{CIB}}$ with $T=[T^{'v}, T^{'l}]$. 
Table~\ref{tab:lxmert_var} shows the VQA-Score for PER on VQA-Rephrasings, revealing that different internal representations have a slight impact on the PER performance of CIB. 
This indicates that for two-stream pretrained VLMs, using the visual and linguistic representations after the vision and language transformer layers as internal representations to estimate the mutual information terms in $\mathcal{L}_{\text{CIB}}$ is also a feasible approach.

\begin{figure}[!t]
\centering
\includegraphics[width=0.85\linewidth]{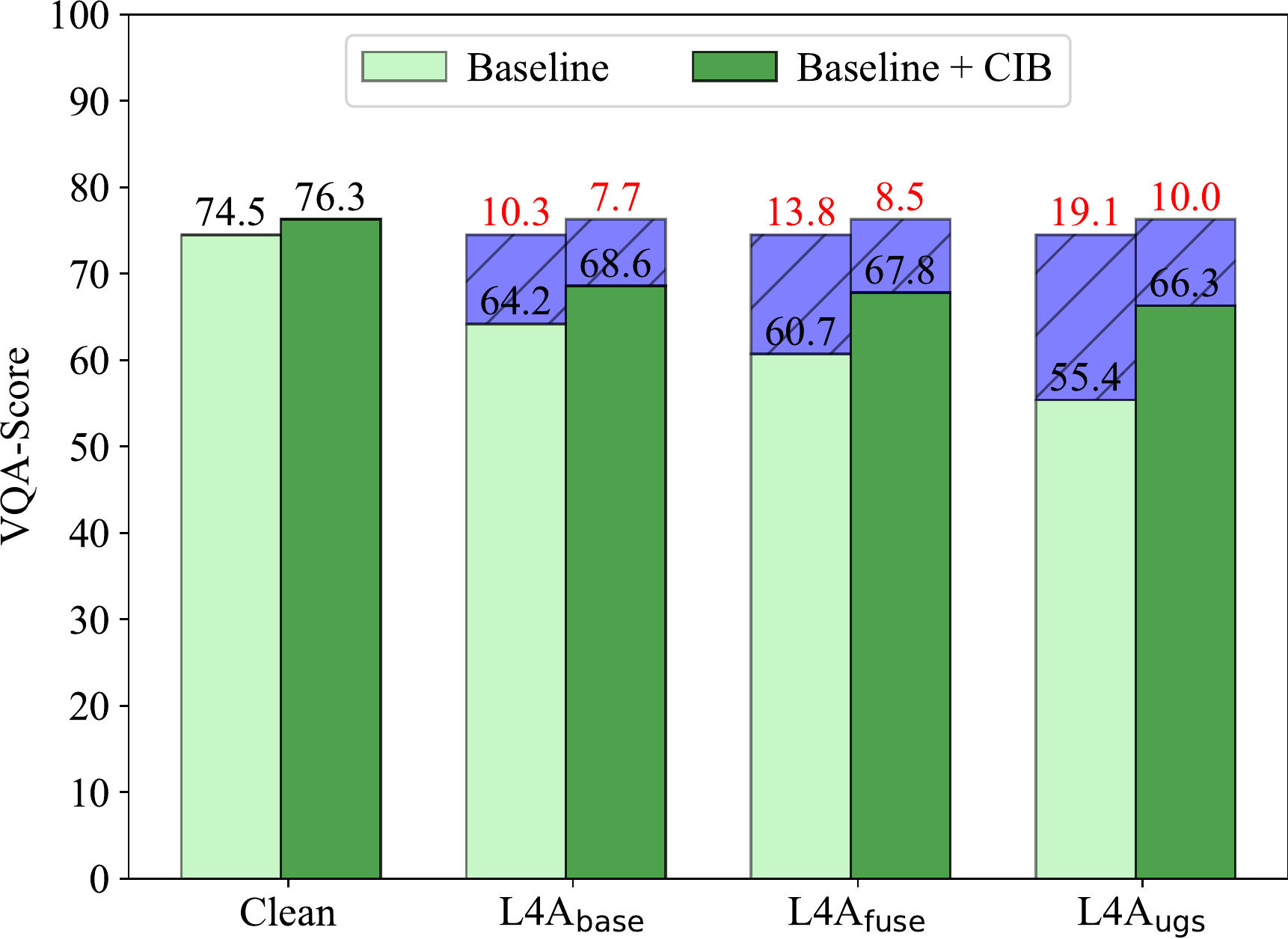}
\vspace{-1mm}
\caption{Results on VQA v2 test-dev under different adversarial attacks}
\label{fig:vis_attack}
\vspace{-2mm}
\end{figure}

\subsection{Discussion and Analysis}

\subsubsection{Effectiveness of CIB for Standard VQA Performance} 

To analyze the impact of CIB on standard VQA performance (\ie whether the representation compression impairs the standard VQA performance), we utilize CIB as the objective to train the aforementioned baseline pretrained VLMs on the VQA v2 training and validation sets. 
The results on VQA v2 test-dev are shown in Table~\ref{tab:comp_vqa_v2_all}. 
\footnote{
Since we do not utilize the additional question-answer pairs from VG \citep{krishna2017visual} for data augmentation in our experiments and some other detail differences, there are minor differences between our re-implementation\textcolor{red}{$^\dagger$} of the baseline pretrained VLMs and the published results in the original papers.} 
Overall, training baseline pretrained VLMs with the proposed CIB can slightly improve the standard VQA performance. 
In particular, the performance improvement of VisualBERT and ALBEF is relatively significant. 
This because that their visual inputs contain more redundant information, such as image background and visual content irrelevant to the given question (VisualBERT and ALBEF respectively adopt 100 region-level features and 900 patch-level features as visual inputs). 
Therefore, our hypothesis is that a certain degree of compression of representations can reduce the redundant information learned from inputs and make the obtained representations more compact and robust. 
\RCRevised{
Noting that, in contrast to the significant improvement in input robustness, CIB leads to relatively limited improvement in standard performance. 
This observation also indirectly indicates that the proposed CIB is carefully designed and tailored to improve input robustness when adapting pretrained VLMs for the downstream VQA. 
}

\begin{figure*}[!t]
\centering
\includegraphics[width=0.99\linewidth]{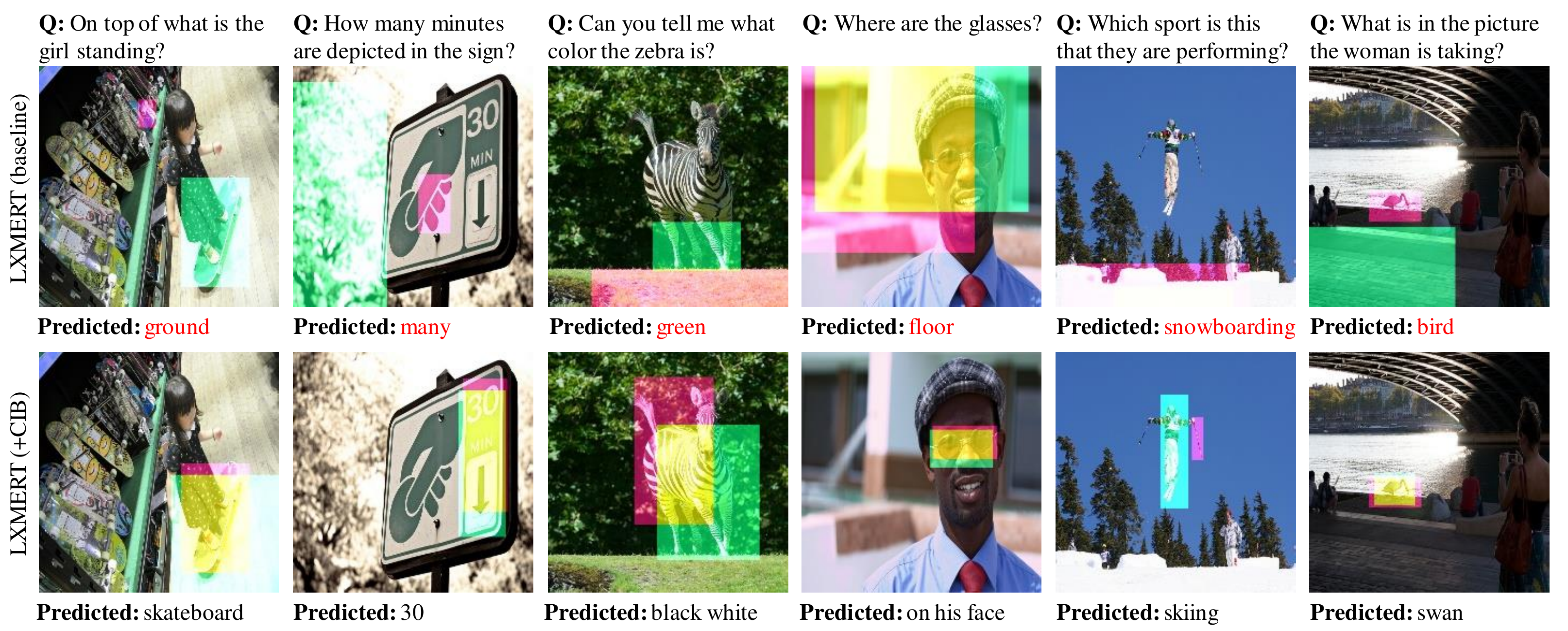}
\caption{Visualization of the top two objects with the highest attention scores. 
The image-question pairs originate from VQA-Rephrasings. 
Objects with the best and second attention scores are marked in \textcolor{magenta}{magenta} and \textcolor{green}{green}. 
The wrongly predicted answers are marked in \textcolor{red}{red}
}
\label{fig:example_hm}
\vspace{-2mm}
\end{figure*}

\begin{table}[!t]
\begin{center}
\caption{Results on AdVQA \citep{sheng2021human}}
\label{tab:all_on_advqa} 
\vspace{-1mm}
\setlength{\tabcolsep}{2.2mm}{
\begin{tabularx}{\linewidth}{@{}lll@{}} 
\toprule
\multirow{2}{*}{Methods}
&\multicolumn{2}{c}{VQA-Score} 
\\ 
\cmidrule(l){2-3}
&test 
&val 
\\ 
\midrule[1pt]
VisualBERT \citep{li2019visualbert} 
&31.96 
&28.09 
\\ 
ViLBERT \citep{lu2019vilbert} 
&32.01
&33.67 
\\ 
ViLT \citep{kim2021vilt} 
&31.00 
&32.48 
\\ 
UNITER$_\text{B}$ \citep{chen2020uniter}
&27.56 
&29.44 
\\ 
VILLA$_\text{B}$ \citep{gan2020large} 
&27.55 
&29.36 
\\ 
UNITER$_\text{L}$ \citep{chen2020uniter} 
&29.66 
&32.08 
\\ 
VILLA$_\text{L}$ \citep{gan2020large} 
&28.59 
&30.58
\\ 
M4C \citep{hu2020iterative}
&36.57 
&36.93
\\ 
\midrule
UNITER$_\text{B}$ \citep{chen2020uniter}\textcolor{red}{$^\dagger$} 
&36.20
&36.73
\\
\quad + CIB 
&37.85 (\positive{+1.65})   
&38.23 (\positive{+1.50})   
\\ 
LXMERT \citep{tan2019lxmert}\textcolor{red}{$^\dagger$} 
&36.30 
&37.09  
\\
\quad + CIB 
&37.42 (\positive{+1.12})  
&39.10 (\positive{+2.01})   
\\ 
\bottomrule 
\end{tabularx}
}
\end{center} 
\vspace{-4mm}
\end{table}

\begin{figure*}[!t]
\centering
\subfigure[Examples of robustness against linguistic variants on VQA-Rephrasings and VQA P2]{
\label{fig:vis_lv}
\begin{minipage}[b]{\textwidth}
\includegraphics[width=0.99\linewidth]{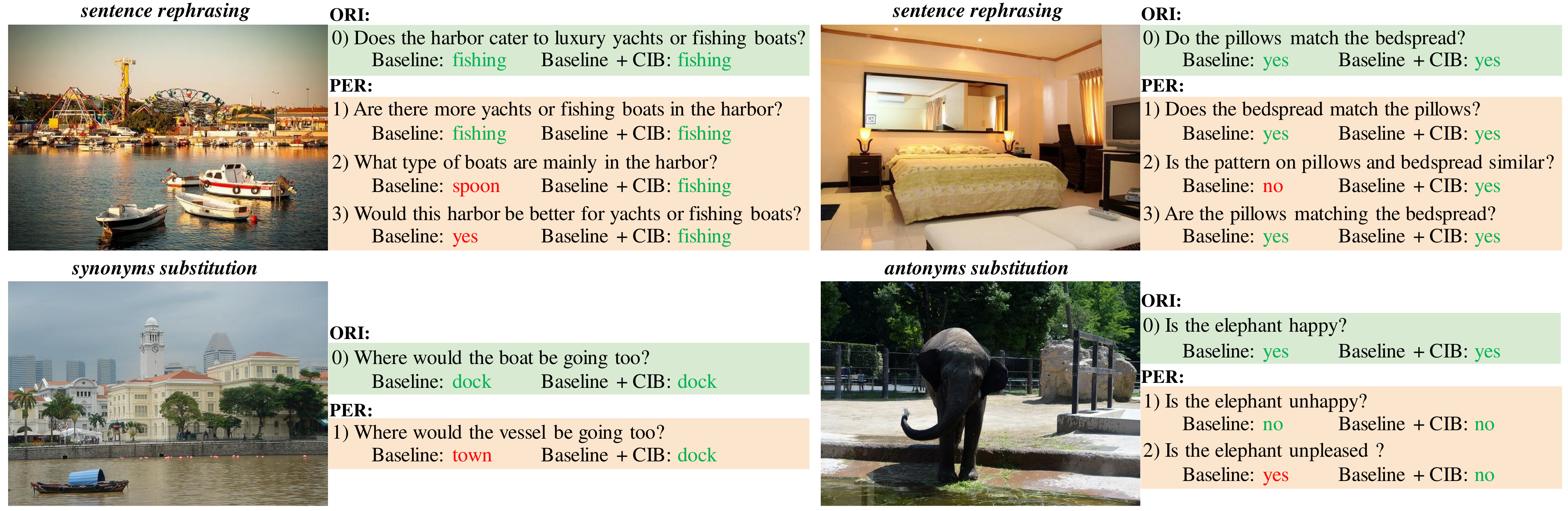}
\end{minipage}
}
\subfigure[Examples of robustness against viaual variants on IV-VQA and CV-VQA]{
\label{fig:vis_vv}
\begin{minipage}[b]{\linewidth}
\includegraphics[width=0.99\linewidth]{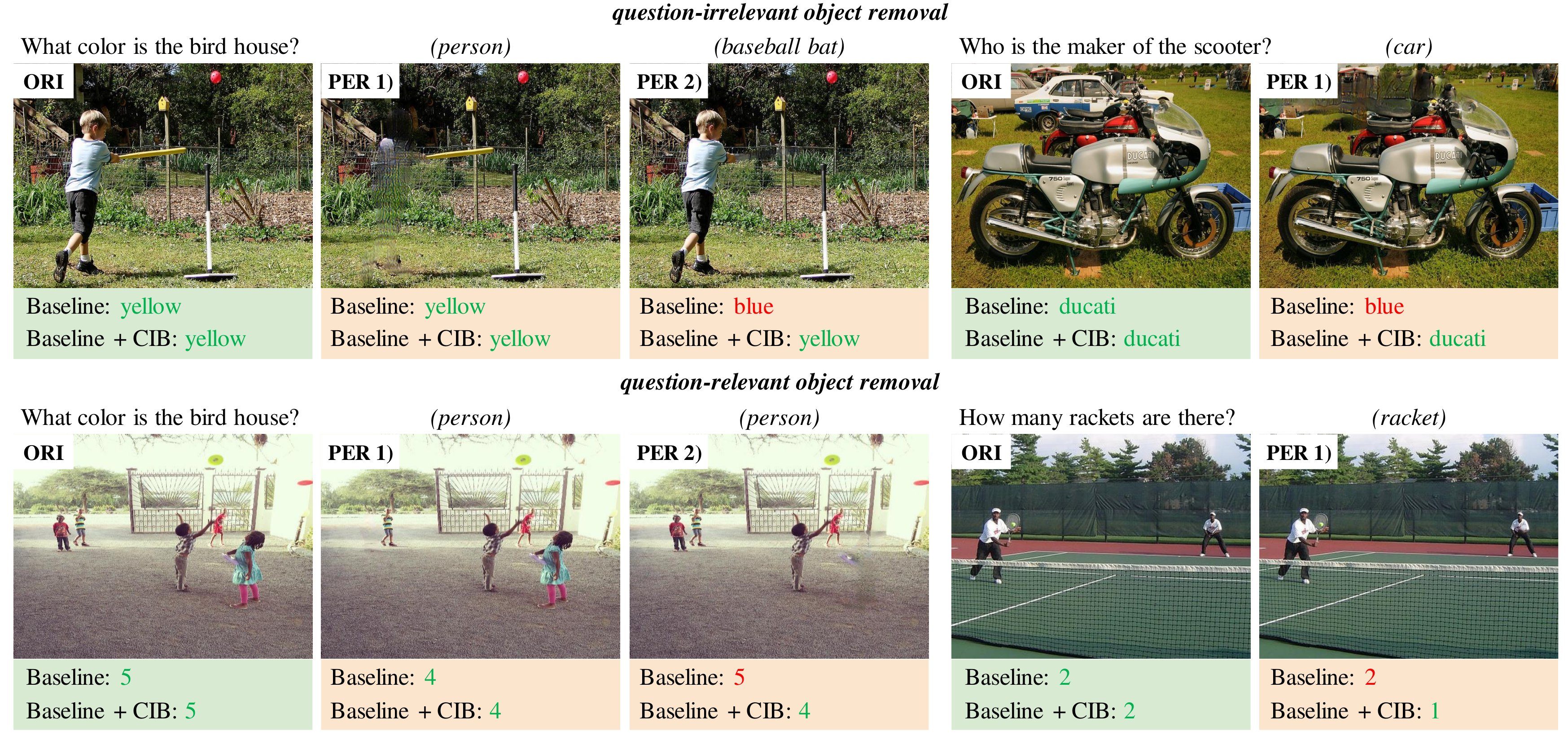}
\end{minipage}
}
\subfigure[Examples of robustness against multimodal shotcut learning on VQA-CE]{
\label{fig:vis_sc}
\begin{minipage}[b]{\linewidth}
\includegraphics[width=0.99\linewidth]{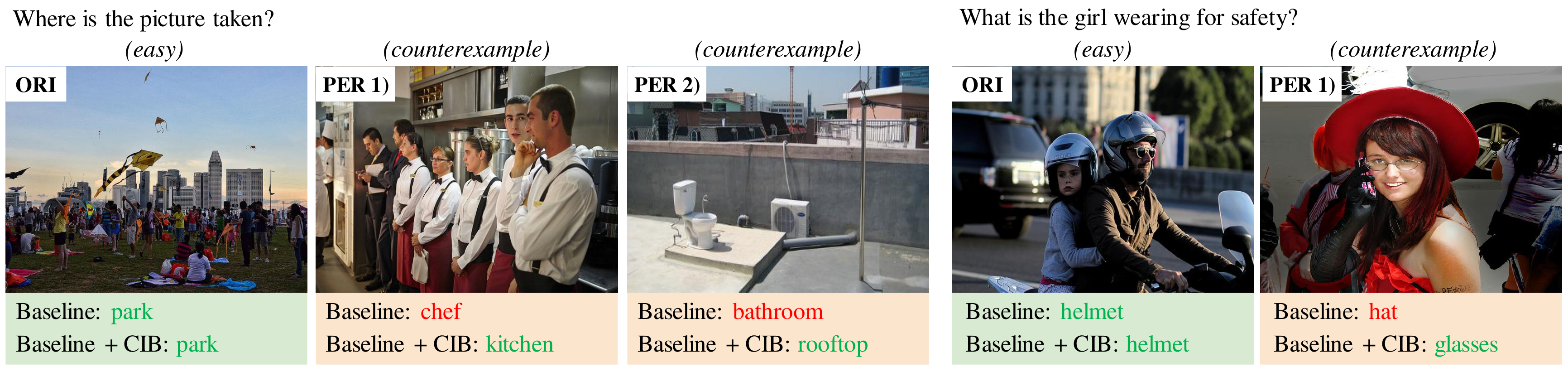}
\end{minipage}
}
\vspace{-2mm}
\caption{
Qualitative examples of the baseline and our method. 
The \textcolor{Green}{correct} and \textcolor{red}{wrong} answers are highlighted
}
\label{fig:vis_example}
\vspace{-4mm}
\end{figure*}

\RCRevised{
\subsubsection{Effectiveness of CIB against Adversarial Attack}

To analyze the impact of CIB on defending against adversarial attacks, we conduct experiments considering the following attacks and dataset: 
(\emph{i}) L4A \citep{ban2022pre}, which adds pretrained adversarial perturbations (PAPs) to the low-level layer of pretrained models, can effectively fool the finetuned models on downstream tasks without any knowledge of the tasks. 
(\emph{ii}) AdVQA \citep{sheng2021human}, an adversarial benchmark collected using a human-and-model-in-the-loop paradigm to attack state-of-the-art VQA models and obtain human-adversarial examples, can effectively evaluate the human-adversarial robustness of VQA models. 
The effectiveness of a model in defending against adversarial attacks is measured by the model VQA-Score under these attacks. 

For (\emph{i}), as proposed in the work \citep{ban2022pre}, we consider three different ways for perturbation generation, \ie L4A$_{\text{base}}$, L4A$_{\text{fuse}}$, and L4A$_{\text{ugs}}$. 
Before finetuning the pretrained ALBEF with a text generation loss and the proposed CIB as a training objective, we first utilize the three methods above to generate PAPs by lifting the neuron activations of low-level layers of the ALBEF. 
Next, we separately add the generated PAPs to input images and finetune the pretrained ALBEF on the VQA v2 training and validation sets, and test their performance on VQA v2 test-dev.  
\figref{vis_attack} shows the performance comparison, where the blue bar marked with red performance indicates the VQA-Score drop with respect to the standard performance under an attack. 
From the figure, we can observe that CIB markedly reduces the performance drop, demonstrating its ability to better alleviate the vulnerability of VQA models to such attacks. 
For (\emph{ii}), following the experimental setups for evaluating input robustness, we first finetune pretrained UNITER$_\text{B}$ and LXMERT on the standard and clean VQA v2 training set, and then evaluate human-adversarial robustness on the adversarial benchmark AdVQA. 
As shown in Table~\ref{tab:all_on_advqa}, the significant performance improvement of our method over baselines demonstrates the robustness of CIB against human-adversarial attacks. 
In summary, the aforementioned experiments consistently suggest that the proposed CIB, as a generic objective, can potentially alleviate the vulnerability of models to adversarial attacks when adapting pretrained VLMs to downstream tasks. 
}

\begin{table}[!t]
\begin{center}
\caption{
Results on the RefCOCO+ \citep{yu2016modeling} dataset for weakly-supervised visual grounding}
\label{tab:vg} 
\vspace{-1mm}
\setlength{\tabcolsep}{5.mm}{
\begin{tabularx}{\linewidth}{@{}lccc@{}}
\toprule
Methods &Val &TestA &TestB
\\
\midrule[1pt]
ARN \citep{liu2019adaptive} 
&32.78 &34.35 &32.13
\\
CCL \citep{zhang2020counterfactual}
&34.29 &36.91 &33.56
\\
ALBEF$_{\text{itc}}$ \citep{nips_LiSGJXH21}
&51.58 &60.09 &40.19
\\
ALBEF$_{\text{itm}}$ \citep{nips_LiSGJXH21}
&58.46 &65.89 &46.25
\\
\quad + CIB
&\textbf{59.41} &\textbf{67.39} &\textbf{47.18}
\\
\bottomrule 
\end{tabularx}
}
\end{center} 
\vspace{-4mm}
\end{table}

\subsubsection{Generalizability of CIB to Other Multimodal Task}
\RCRevised{The proposed CIB is essentially a generic training objective that can be applied to various multimodal tasks beyond the VQA task.}
To evaluate the generalizability of CIB to other multimodal tasks, we consider the task of weakly-supervised visual grounding. 
Following the original experimental setups of ALBEF \citep{nips_LiSGJXH21}, we finetune pretrained ALBEF with CIB on the RefCOCO+ \citep{yu2016modeling} training dataset in a weakly-supervised setups, \ie finetuning models using only image-text supervision without bounding box annotations. 
From the results in Table~\ref{tab:vg}, we find that using CIB as the training objective can further improve the performance of the baseline ALBEF$_{\text{itm}}$, demonstrating that the proposed CIB can be effectively applied to other multimodal tasks.

\subsection{Qualitative Results}

\subsubsection{Visualization of Visual Attentional Objects} 

To empirically explore why CIB can improve input robustness, we utilize the pretrained LXMERT \citep{tan2019lxmert} as a representative and conduct the following experiments. 
First, we enumerate the image-question pairs, whose answers are correctly predicted by the LXMERT finetuned with CIB but incorrectly predicted by the baseline LXMERT finetuned without CIB, from the VQA-Rephrasings dataset. 
Next, we compute the attention score between the final representation $Z \in \mathbb{R}^d$ used for answer prediction and the input visual representation $X^v\in \mathbb{R}^{K\times d}$ of object regions using the formula, \ie $\text{score}_\text{attn} = \softmax(Z\cdot (X^{v})^{\mathbb{T}}/\sqrt{d})$. 
Finally, we utilize the \textcolor{magenta}{magenta} and \textcolor{green}{green} color to highlight the top two objects with the highest attention scores in the image. 
The results in \figref{example_hm} show that compared with the baseline LXMERT, the attended two objects obtained by the LXMERT finetuned with CIB are more consistent and question-related. 
This observation qualitatively illustrates that using CIB as a training objective to finetune pretrained VLMs can encourage models to learn more discriminative representations for different answers and reduce the irrelevant information to questions.

\RCRevised{
\subsubsection{Visualization of Input Robustness Cases}

\figref{vis_lv}, \ref{fig:vis_vv}, and \ref{fig:vis_sc} present several qualitative examples demonstrating robustness to linguistic variations, visual variations, and multimodal shortcut learning, respectively. 
According to the qualitative comparison in \figref{vis_example}, we can further observe that compared to the baseline that finetunes the pretrained LXMERT with a cross-entropy loss, using the proposed CIB as a training objective to finetune the pretrained VLM can improve the ability of VQA models to correctly answer these difficult questions. 
This empirical evidence highlights the effectiveness of CIB in defending against such attacks involving both visual and linguistic inputs. 
}

\section{Conclusion} 
\label{sec:cd}

In this paper, we propose to improve input robustness from the information bottleneck perspective when adapting pretrained VLMs to the downstream VQA task. 
Specifically, we derive a new IB lower bound (CIB) for vision-language learning and apply CIB to finetune pretrained VLMs with various architectures for VQA. 
Extensive experiments on five robustness datasets consistently demonstrate the effectiveness and superiority of CIB. 
In the future, we plan to assess the effectiveness of CIB when tuning pretrained VLMs using parameter-effective strategies, such as adapter-based tuning and prompt-based tuning. 

\noindent{\bf Limitation.} 
Redundancy has two sides. 
One reason why pretrained VLMs can significantly improve the performance of downstream tasks is that they have learned rich and redundant knowledge during the pretraining stage. 
Practically, for downstream tasks, especially in-domain tasks, task-related redundancy can help models quickly adapt to new tasks, while task-agnostic redundancy may impair model robustness. 
Our work investigates improving input robustness of models while preserving their accuracy by seeking a tradeoff between representation compression and redundancy. 
Another potential research direction is to explore how to explicitly reduce task-agnostic redundancy and adequately exploit task-related redundancy when adapting pretrained VLMs to downstream tasks, particularly out-of-domain tasks.

\begin{acknowledgement}
This work was supported by the National Science Foundation of China (Grant No. 62088102). 
\end{acknowledgement}




\bibliographystyle{spbasic}
{\footnotesize
\bibliography{Reference}
}

\appendix
\section{Appendix}
\label{sec:app_proof}

\subsection{Proof for Theorem 1} 
\label{sec:app_proof_theorem1}

To prove the Theorem~\ref{thm:JMI} stated in Section~\ref{sec:CIB}, we first enumerate some properties of mutual information (MI). 
Specifically, for any random variables $X$, $Y$ and $Z$, we have: 
\begin{enumerate}[itemindent=1.2em]
\item[$(P_1)$] Positivity: 
\begin{align*}
I(X; Y) \ge 0, I(X; Y|Z)\ge 0. 
\end{align*}
\item[$(P_2)$] Chain rule: 
\begin{align*}
&I(X, Y; Z) \\
&= I(Y; Z) + I(X; Z|Y),\\
&= I(X; Z) + I(Y; Z|X),\\
&= \frac{1}{2}\left[I(Y; Z) + I(X; Z) + I(X; Z|Y) + I(Y; Z|X)\right]. 
\end{align*}
\item[$(P_3)$] Chain rule (Multivariate Mutual Information):
\begin{align*}
I(X; Y; Z) = I(Y; Z) - I(Y; Z|X). 
\end{align*}
\item[$(P_4)$] Positivity of discrete entropy (for discrete $X$): 
\begin{align*}
H(X)\ge 0, H(X|Y)\ge 0. 
\end{align*}
\item[$(P_5)$] Entropy and Mutual Information: 
\begin{align*}
H(X)= H(X|Y)+ I(X; Y). 
\end{align*}
\end{enumerate}

\noindent Then, we state the following three easily provable lemmas: 

\begin{lemma}
\label{lem:mi_xz_condition}
In representation learning, given a random variable $X$, the random variable $Z$ is defined to be a representation of $X$, we can simply state that $Z$ is conditionally independent from any other variable in the model once $X$ is observed. 
That is, for any variable (or groups of variables) $T_1$ and $T_2$ in the model, we have 
\begin{align*} 
I(Z; T_1|X, T_2) = 0.
\end{align*}  
\end{lemma}

\begin{lemma}
\label{lem:mi_xz_mutual}
Given a sequence of random variables $X_1$, $X_2$, $\dots$, $X_n$ and a deterministic function $f$, then $\forall \, i, j = 1, 2, \dots, n$, we have 
\begin{align*} 
I(X_i; f(X_i)) \ge I(X_j; f(X_i)). 
\end{align*}  
\end{lemma}
\begin{proof}
By the definition, 
\begin{align*}
I(X_i; f(X_i)) &= H(f(X_i)) - H(f(X_i) \mid X_i), \\
I(X_j; f(X_i)) &= H(f(X_i)) - H(f(X_i) \mid X_j). 
\end{align*}
Since $f$ is a deterministic function, 
\begin{align*}
H(f(X_i) \mid X_i) &= 0, \\
H(f(X_i) \mid X_j) &\ge 0. 
\end{align*}
Therefore, 
\begin{align*}
I(X_i; f(X_i)) \ge I(X_j; f(X_i)).
\end{align*}
\end{proof}

\begin{lemma}
Let $Z_1$ and $Z_2$ are the representations of $X_1$ and $X_2$, then 
\begin{align*} 
&I_\theta(X_1;Z_1|X_2) \le \KL\left(p_\theta(Z_1|X_1)||p_\psi(Z_2|X_2)\right), \\
&I_\psi(X_2;Z_2|X_1) \le \KL\left(p_\psi(Z_2|X_2)||p_\theta(Z_1|X_1)\right). 
\end{align*}  
\end{lemma} 
\begin{proof}
By the definition, 
\begin{align*}
&I_\theta(X_1; Z_1|X_2)\\ 
&= \E_{x_1, x_2\sim p(X_1, X_2)} \E_{z\sim p_\theta(Z_1|X_1)}\left[\log\frac{p_\theta(Z_1 = z|X_1=x_1)}{p_\theta(Z_1=z|X_2=x_2)}\right],\\
&=\E_{x_1, x_2\sim p(X_1, X_2)}\E_{z\sim p_\theta(Z_1|X_1)}\left[\log\frac{p_\theta(Z_1=z|X_1=x_1)}{p_\psi(Z_2=z|X_2=x_2)}
\right] \\ 
&\, - \E_{x_1, x_2\sim p(X_1, X_2)}\E_{z\sim p_\theta(Z_1|X_1)}\left[\log\frac{p_\theta(Z_1=z|X_2=x_2)}{p_\psi(Z_2=z|X_2=x_2)}
\right],
\\ 
&=\KL\left(p_\theta(Z_1|X_1)||p_\psi(Z_2|X_2)\right) - \KL\left(p_\theta(Z_2|X_1)||p_\psi(Z_2|X_2)\right),\\ 
&\le \KL\left(p_\theta(Z_1|X_1)||p_\psi(Z_2|X_2)\right). 
\end{align*} 
If and only if $p_\psi(Z_2|X_2)$ coincides with $p_\theta(Z_1|X_2)$, the equality holds. 
Analogously, $I_\psi(X_2; Z_2|X_1)\le \KL(p_\psi(Z_2|X_2)||p_\theta(Z_1|X_1))$ is proved. 
\end{proof}


\noindent Next, we utilize the above properties and lemmas to prove  Theorem~\ref{thm:JMI}. 

\vspace{2mm}

\noindent\textbf{Theorem 1} 
\textit{
(Upper Bound of $I(X^v, X^l; T^v, T^l)$) 
Given two groups of random variables $X=[X^v, X^l]$ and $T=[T^v, T^l]$, $I(X^v, X^l; T^v, T^l)$ can be upper-bounded with 
\begin{align}
I(X; T) &= I(X^v, X^l; T^v, T^l), \notag \\
&\le I(X^v; T^v) + I(X^l; T^l) \textcolor{magenta}{- I(T^v; T^l)} + \SKL, \notag
\end{align}
where $\SKL$ denotes the symmetric Kullback-Leibler (KL) divergence and can be obtained by averaging the divergences $\KL(p(t^v|x^v)||p(t^l|x^l))$ and $\KL(p(t^l|x^l)||p(t^v|x^v))$. 
}

\begin{proof}
\begin{align*}
&I(X; T) \\
&= I(X^l, X^v; T),\\ 
&\stackrel{(P_2)}{=}\frac{1}{2}\left[I(X^l; T) + I(X^v; T) + I(X^l; T|X^v) + I(X^v; T|X^l)\right],\\
&=\frac{1}{2}\left[I(X^l; T^l, T^v) + I(X^v; T^l, T^v) + I(X^l; T|X^v) + I(X^v; T|X^l)\right].
\end{align*}
Since, 
\begin{align*}
&I(X^l; T^l, T^v) \\
&\stackrel{(P_2)}{=}I(X^l; T^l) + I(X^l; T^v|T^l), \\ 
&\stackrel{(P_3)}{=}I(X^l; T^l) + I(X^l; T^v) - I(X^l; T^v; T^l),\\
&\stackrel{(P_3)}{=}I(X^l; T^l) + I(X^l; T^v) - I(T^l; T^v) + I(T^l; T^v|X^l),\\
&\stackrel{(LA 1)}{=}I(X^l; T^l) + I(X^l; T^v) - I(T^l; T^v),\\ 
&\stackrel{(LA 2)}{\le}2I(X^l; T^l) - I(T^l; T^v). 
\end{align*}
Analogously, $I(X^v; T^l, T^v)$ is upper bounded by  
\begin{align*}
I(X^v; T^l, T^v) \le 2I(X^v; T^v) - I(T^l; T^v). 
\end{align*} 
And, 
\begin{align*}
&I(X^l; T|X^v) \\
&= I(X^l; T^l, T^v|X^v),\\
&\stackrel{(P_2)}{=} I(X^l; T^l|X^v) + I(X^l; T^v|X^v,T^l), \\ 
&\stackrel{(LA 1)}{=} I(X^l; T^l|X^v); \\ 
&I(X^v; T|X^l) \\
&= I(X^v; T^l, T^v|X^l),\\
&\stackrel{(P_2)}{=} I(X^v; T^v|X^l) + I(X^v; T^l|X^l,T^v), \\ 
&\stackrel{(LA 1)}{=} I(X^v; T^v|X^l). \\ 
\end{align*}
Let $\SKL=\frac{1}{2}\left(\KL(p_\theta||p_\psi) + \KL(p_\psi||p_\theta)\right)$, \\ 
therefore, 
\begin{align*}
&I(X; T)\\ 
&= I(X^l, X^v; T^l, T^v),\\
&\le I(X^l; T^l) + I(X^v; T^v) - I(T^l; T^v) + \frac{1}{2}\left[I(X^l; T^l|X^v) + I(X^v; T^v|X^l)\right],\\ 
&\stackrel{(LA 3)}{\le} I(X^l; T^l) + I(X^v; T^v) - I(T^l; T^v) \\
&\qquad + \frac{1}{2}\left[\KL\left(p_\theta(T_1|X_1)||p_\psi(T_2|X_2)\right) + \KL\left(p_\psi(T_2|X_2)||p_\theta(T_1|X_1)\right)\right],\\
&=I(X^l; T^l) + I(X^v; T^v) - I(T^l; T^v) + \SKL. 
\end{align*}
\end{proof}

\subsection{Proof for Alternative Upper Bound} 
\label{sec:alternative_ub}

The alternative three upper bounds of $I(X^v, X^l; T^v, T^l)$ utilized in Section~\ref{sec:abl_cib} are derived as follows:

\begin{theorem} (Upper Bound of $I(X^v, X^l; T^v, T^l)$) 
Given two groups of random variables $X=[X^v, X^l]$ and $T=[T^v, T^l]$, the mutual information $I(X^v, X^l; T^v, T^l)$ can be upper-bounded with 
\begin{align*}
I(X^v, X^l; T^v, T^l) \le \frac{3}{2}\left[I(X^v; T^v) + I(X^l; T^l)\right]. 
\end{align*}
\label{thm:JMI_a1}
\end{theorem} 
\begin{proof}
\begin{align*}
&I(X^l, X^v; T^l, T^v) \\
&\le I(X^l; T^l) + I(X^v; T^v) - I(T^l; T^v) + \frac{1}{2}\left[I(X^l; T^l|X^v) + I(X^v; T^v|X^l)\right],\\ 
&\le I(X^l; T^l) + I(X^v; T^v) + \frac{1}{2}\left[I(X^l; T^l|X^v) + I(X^v; T^v|X^l)\right],\\ 
&\le I(X^l; T^l) + I(X^v; T^v) + \frac{1}{2}\left[I(X^l; T^l) + I(X^v; T^v)\right],\\ 
&= \frac{3}{2}\left[I(X^l; T^l) + I(X^v; T^v)\right]. 
\end{align*}
\end{proof}

\begin{theorem} (Upper Bound of $I(X^v, X^l; T^v, T^l)$) 
Given two groups of random variables $X=[X^v, X^l]$ and $T=[T^v, T^l]$, the mutual information $I(X^v, X^l; T^v, T^l)$ can be upper-bounded with 
\begin{align*}
I(X^v, X^l; T^v, T^l) \le I(X^v; T^v) + I(X^l; T^l) + \SKL. 
\end{align*}
where $\SKL$ denotes symmetric Kullback-Leibler (KL) divergence and can be obtained by averaging the divergences $\KL(p(t^v|x^v)||p(t^l|x^l))$ and $\KL(p(t^l|x^l)||p(t^v|x^v))$. 
\label{thm:JMI_a2}
\end{theorem} 
\begin{proof}
\begin{align*}
&I(X^l, X^v; T^l, T^v) \\
&\stackrel{(\text{Theorem 1})}{\le} I(X^l; T^l) + I(X^v; T^v) - I(T^l; T^v) + \SKL, \\
&\stackrel{(I(T^l; T^v) \ge 0)}{\le} I(X^l; T^l) + I(X^v; T^v)  + \SKL. 
\end{align*}
\end{proof}

\begin{theorem} (Upper Bound of $I(X^v, X^l; T^v, T^l)$) 
Given two groups of random variables $X=[X^v, X^l]$ and $T=[T^v, T^l]$, the mutual information $I(X^v, X^l; T^v, T^l)$ can be upper-bounded with 
\begin{align*}
I(X^v, X^l; T^v, T^l) \le -I(T^v; T^l) + \SKL. 
\end{align*}
where $\SKL$ denotes symmetric Kullback-Leibler (KL) divergence and can be obtained by averaging the divergences $\KL(p(t^v|x^v)||p(t^l|x^l))$ and $\KL(p(t^l|x^l)||p(t^v|x^v))$. 
\label{thm:JMI_a3}
\end{theorem} 
Please see the work of Federici~\etal \citep{federici2020learning} for proof.

\RCRevised{
\subsection{Proof for Theoretical Justification of Input Robustness}
\label{sec:proof_performance_bound}

Finally, we prove the theoretical justification for the input robustness of CIB in Eq.~\eqref{eq:justification}, \ie the following inequality: 
\begin{align*}
&|I(T; Y) - I(T'; Y)| \\
&= |I(T^v, T^l; Y) - I({T^v}', {T^l}'; Y)|, \notag \\ 
&\leq B^v_1 \sqrt{\mathcal{T}^v} \left(I(X^v;T^v)\right)^{1/2} + 
B^v_2 |\mathcal{T}^v|^{3/4} \left(I(X^v;T^v)\right)^{1/4} \notag \\
&~~~ + B^v_3 \sqrt{|\mathcal{T}^v|} \left(I({X^v}';{T^v}')\right)^{1/2} + 
B^v_4 |\mathcal{T}^v|^{3/4} \left(I({X^v}';{T^v}')\right)^{1/4} \notag \\
&~~~ + B^l_1 \sqrt{\mathcal{T}^l} \left(I(X^l;T^l)\right)^{1/2} + 
B^l_2 |\mathcal{T}^l|^{3/4} \left(I(X^l;T^l)\right)^{1/4} \notag \\
&~~~ + B^l_3 \sqrt{|\mathcal{T}^l|} \left(I({X^l}';{T^l}')\right)^{1/2} + 
B^l_4 |\mathcal{T}^l|^{3/4} \left(I({X^l}';{T^l}')\right)^{1/4} \notag \\
&~~~ + B^v_0 + B^l_0, 
\end{align*}
where $\mathcal{T}^v$ is the finite support of $T^v$ and ${T^v}'$, and $B^v_0$, $B^v_1$, $B^v_2$, $B^v_3$, and $B^v_4$ are constants that depend on the sequence length $K$, $\delta$, and $p(x^v)$. 
$\mathcal{T}^l$ is the finite support of $T^l$ and ${T^l}'$, and $B^l_0$, $B^l_1$, $B^l_2$, $B^l_3$, and $B^l_4$ are constants that depend on the sequence length $L$, $\delta$, and $p(x^l)$.

\begin{proof}
According to triangle inequality and data processing inequality:
\begin{align*}
&|I(T; Y) - I(T'; Y)| \\
&= |I(T^v, T^l; Y) - I({T^v}', {T^l}'; Y)|, \\
&= |I(T^v; Y) + I(T^l; Y | T^v) - I({T^v}'; Y) - I({T^l}'; Y | {T^v}')|, \\
&= |I(T^v;Y) - I({T^v}';Y)| + |I(T^l;Y|T^v) - I({T^l}';Y|{T^v}')|, \\
&\leq |I(T^v;Y) - I({T^v}';Y)| + |I(T^l;Y) - I({T^l}';Y)|. 
\end{align*} 
Then, we can further approximate each of the two terms on the upper bound separately. 
For the first term, using the bound in the work \citep{wang2021infobert}, we obtain the following upper bound:
\begin{align*}
&|I(T^v;Y) - I({T^v}';Y)| \\ 
&\leq B^v_0 + B^v_1 \sqrt{\mathcal{T}^v} \left(I(X^v;T^v)\right)^{1/2} + 
B^v_2 \left|\mathcal{T}^v\right|^{3/4} \left(I(X^v;T^v)\right)^{1/4} \\
&~~~~ + B^v_3 \sqrt{|\mathcal{T}^v|} \left(I({X^v}';{T^v}')\right)^{1/2} + 
B^v_4 \left|\mathcal{T}^v\right|^{3/4} \left(I({X^v}';{T^v}')\right)^{1/4}, 
\end{align*} 
where $\mathcal{T}^v$ is the finite support of $T^v$ and ${T^v}'$, and $B^v_0$, $B^v_1$, $B^v_2$, $B^v_3$, and $B^v_4$ are constants that depend on the sequence length $K$, $\delta$, and $p(x^v)$. 
Analogously, the second term can be bounded by: 
\begin{align*}
&|I(T^l;Y) - I({T^l}';Y)| \\ 
&\leq B^l_0 + B^l_1 \sqrt{\mathcal{T}^l} \left(I(X^l;T^l)\right)^{1/2} + 
B^l_2 \left|\mathcal{T}^l\right|^{3/4} \left(I(X^l;T^l)\right)^{1/4} \\
&~~~~ + B^l_3 \sqrt{|\mathcal{T}^l|} \left(I({X^l}';{T^l}')\right)^{1/2} + 
B^l_4 \left|\mathcal{T}^l\right|^{3/4} \left(I({X^l}';{T^l}')\right)^{1/4}, 
\end{align*}
where $\mathcal{T}^l$ is the finite support of $T^l$ and ${T^l}'$, and $B^l_0$, $B^l_1$, $B^l_2$, $B^l_3$, and $B^l_4$ are constants that depend on the sequence length $L$, $\delta$, and $p(x^l)$. 
Combining the above two terms, Eq.~\eqref{eq:justification} is proved. 
\end{proof}
}

\end{document}